%% file: arxiv.tex
\documentclass{article}

\PassOptionsToPackage{numbers, compress}{natbib}
\usepackage[preprint]{neurips_2022_arxiv}

\usepackage{def}
\usepackage[utf8]{inputenc} % allow utf-8 input
\usepackage[T1]{fontenc}    % use 8-bit T1 fonts
\usepackage[colorlinks=true, linkcolor=blue, citecolor=blue]{hyperref}       % hyperlinks
\usepackage{url}            % simple URL typesetting
\usepackage{booktabs}       % professional-quality tables
\usepackage{amsfonts}       % blackboard math symbols
\usepackage{nicefrac}       % compact symbols for 1/2, etc.
\usepackage{microtype}      % microtypography
\usepackage{xcolor}         % colors
\usepackage{algorithm}
\usepackage{algorithmic}
\usepackage{comment}
\usepackage{enumitem}
\usepackage{subcaption}
\usepackage{graphicx}
\input{math_commands}

\newtheorem{lem}{Lemma}[section]
\newtheorem{assumption}{Assumption}

\title{Provably Efficient Reinforcement Learning for Online Adaptive Influence Maximization}

\author{%
Kaixuan Huang$^{1*}$ \quad Yu Wu$^{1*}$ \quad Xuezhou Zhang$^{1}$ \quad Shenyinying Tu$^{2}$ \\ \quad \textbf{Qingyun Wu}$^{3}$ \quad \textbf{Mengdi Wang}$^{1}$ \quad \textbf{Huazheng Wang}$^{1}$\\
$^{1}$Princeton University \quad $^{2}$LinkedIn \quad $^{3}$Penn State University\\
\texttt{\{kaixuanh, yuw, xz7392, mengdiw, huazheng.wang\}@princeton.edu}\\
\texttt{stu@linkedin.biz}, \texttt{qingyun.wu@psu.edu}
}

\definecolor{darkgreen}{rgb}{0,0.5,0}
\definecolor{darkred}{rgb}{0.7,0,0}
\definecolor{teal}{rgb}{0.3,0.8,0.8}
\definecolor{orange}{rgb}{1.0,0.5,0.0}
\definecolor{purple}{rgb}{0.8,0.0,0.8}
\definecolor{OliveGreen}{rgb}{0.7,0.7,0.3}

\begin{document}
\maketitle
\def\thefootnote{*}\footnotetext{Equal Contribution.}

\begin{abstract}
\input{sections/abstract} 
\end{abstract}

\input{sections/intro}

\input{sections/theory}
\input{sections/algorithm}

\input{sections/reg}

\input{sections/experiment}
\input{sections/conclusion}

\bibliographystyle{plainnat}
\bibliography{reference}

%%%%%%%%%%%%%%%%%%

% \input{checklist} % Must complete and include! 
\clearpage
\appendix

%\label{append:proof} % TO BE Deleted!

\input{appendix/self-normalized}

\input{appendix/regret}

\input{appendix/aux_lemma}
\input{appendix/appendix_experiments}

% \clearpage

\input{appendix/generalized_linear}
\end{document}

%% file: math_commands.tex
\usepackage{amsmath,amsfonts,bm}

\def\eqref#1{(\ref{#1})}
\def\1{\bm{1}}

\DeclareMathAlphabet{\mathsfit}{\encodingdefault}{\sfdefault}{m}{sl}
\SetMathAlphabet{\mathsfit}{bold}{\encodingdefault}{\sfdefault}{bx}{n}

\def\gF{{\mathcal{F}}}
\def\gG{{\mathcal{G}}}

\def\gM{{\mathcal{M}}}

\def\gO{{\mathcal{O}}}

\def\gS{{\mathcal{S}}}
\def\gT{{\mathcal{T}}}

\newcommand{\E}{\mathbb{E}}

\newcommand{\reg}{\mathrm{Regret}}

%% file: sections/abstract.tex
 
Online influence maximization aims to maximize the influence spread of a content in a social network with unknown network model by selecting a few seed nodes. Recent studies followed a non-adaptive setting, where the seed nodes are selected before the start of the diffusion process and network parameters are updated when the diffusion stops.
We consider an adaptive version of content-dependent online influence maximization problem where the seed nodes are sequentially activated based on real-time feedback. In this paper, we formulate the problem as an infinite-horizon discounted MDP under a linear diffusion process and present a model-based reinforcement learning solution. Our algorithm maintains a network model estimate and selects seed users adaptively, exploring the social network while improving the optimal policy optimistically. We establish $\widetilde \gO(\sqrt{T})$ regret bound for our algorithm. Empirical evaluations on synthetic network demonstrate the efficiency of our algorithm. 
 

%% file: sections/intro.tex
\section{Introduction}

Influence Maximization (IM) \cite{kempe2003maximizing,kitsak2010identification,centola2007complex}, motivated by real-world social-network applications such as viral marketing, has been extensively studied in the past decades. In viral marketing, a marketer selects a set of users (seed nodes) with significant influence for content promotion. These selected users are expected to influence their social network neighbors, and such influence will be propagated across the network. With limited seed nodes, the goal of IM is to maximize the information spread over the network. A typical IM formulation models the social network as a directed graph and the associated edge weights are the propagation probabilities across users. Influence propagation is commonly modeled by a certain stochastic diffusion process, such as independent cascade (IC) model and linear threshold (LT) model \cite{kempe2003maximizing}. A popular variant is topic-aware IM \cite{chen2015online, chen2016real} where the activation probabilities are content-dependent and personalized, i.e., edge weights are different when propagating different contents. 

Classical influence maximization solutions are studied in an offline setting, assuming activation probabilities are given \cite{kempe2003maximizing,chen2009efficient, chen2010scalable}. However, this information may not be fully observable in many real-world applications. Online influence maximization \cite{chen2013combinatorial, wen2017online, vaswani2017model} has recently attracted significant attention to tackle this problem, where an agent learns the activation probabilities by repeatedly interacting with the network. Most existing works on online influence maximization are formulated as a multi-armed bandits problem making a \emph{non-adaptive} batch decision: at each round, the seed nodes are computed prior to the diffusion process by balancing exploring the unknown network and maximizing the influence spread; the agent observes either edge-level \cite{chen2013combinatorial, wen2017online, wu2019factorization} or node-level \cite{vaswani2017model, li2020online} activations when the diffusion finishes and updates its model. Combinatorial multi-armed bandits \cite{chen2013combinatorial, wang2017improving} and combinatorial linear bandits \cite{wen2017online, wu2019factorization} algorithms have been proposed as solutions, where most works follow independent cascade model with edge-level feedback.

In contrast to the non-adaptive setting, adaptive influence maximization allows the agent to select seed nodes in a sequential manner after observing partial diffusion results \cite{golovin2011adaptive, tong2016adaptive,peng2019adaptive}. The agent can achieve a higher influence spread since the decision adapts to the real-time feedback of diffusion. In viral marketing, the agent could observe partial diffusion feedback from the customer and adjust the campaign for the rest of budgets based on current diffusion state. Unfortunately, online influence maximization in an adaptive setting is under-explored. Previous bandit-based solutions cannot be applied because the decisions of bandit algorithms are independent of the network state.

In this paper, we study the content-dependent \emph{online adaptive influence maximization} problem: at each round, the agent selects a user-content pair to activate based on current network state, observes the immediate diffusion feedback, and updates its policy in real-time. The network's activation probabilities are content-dependent and are unknown to the agent. The agent's goal is to maximize the total influence spread. We formulate this problem as an infinite-horizon discounted Markov decision process (MDP), where the state is users' current activation status under different contents (user-content pairs), an action is to pick a user-content pair as the new seed, and the total reward is the discounted sum of active user counts.
Specifically, we study the problem under the independent cascade model with node-level feedback. Similar to combinatorial linear bandits \cite{wen2017online, vaswani2017model}, we formulate a tensor network diffusion process where activation probabilities are assumed to be linear with respect to both user and content features. To tackle the problem of node-level feedback, we propose a Bernoulli independent cascade model, a linear approximation to the classic IC model which requires edge-level feedback to learn. 

We propose a model-based reinforcement learning (RL) algorithm to learn the optimal adaptive policy. Our approach builds on prior work of bandit-based influence maximization algorithms \cite{chen2013combinatorial, wen2017online, vaswani2017model} and has the following distinct features: (1) Our adaptive IM policy makes decisions and updates policy on the fly, without waiting till the end of diffusion process; (2) Our algorithm takes into consideration real-time feedback from the network, thus approaching a dynamic-optimal policy and outperforming bandit-based static-optimal solutions; (3) Our algorithm learns from node-level feedback, which greatly relaxes the common edge-level feedback assumption in previous works with IC model; (4) Our policy can handle content-dependent networks and select the best content for the right user for the campaign; (5) To improve computation efficiency, we adopt the low switching cost strategy \cite{abbasi2011improved} that only update model parameter for $\gO(d\log T)$ times, where $d$ is the feature space dimension.  
Our contributions are summarized as follows:
\begin{itemize}[leftmargin=*,itemsep=0pt]
    \item We propose a linear tensor diffusion model for content propagation in social networks and formulate the problem as an infinite-horizon discounted MDP.
    \item We propose a tensor-regression-based RL influence maximization algorithm with optimistic planning that learns an adaptive policy from node-level feedback, which selects the content and next seed user based on current state of the network.
    \item We proved a $\widetilde \gO( d\sqrt{T}/\Delta+ \sqrt{dNKT})$\footnote{$\widetilde \gO(\cdot)$ ignores all logarithmic terms.} regret of our algorithm, where $T$ is the total rounds, $N$ is the number of users, $K$ is the number of contents, $\Delta$ is the coefficient for diffusion decay, $d$ is the dimension related to user and content feature. To our best knowledge, this is the first sublinear regret bound for online adaptive influence maximization. 
    \item We empirically validated on synthetic network that our algorithm explores the unknown network more thoroughly than conventional bandit methods, achieving larger influence spread.  
\end{itemize}

\paragraph{Related Works.} The classical works on (offline) influence maximization \citep{kempe2003maximizing, chen2009efficient,chen2010scalable} assume the network model, i.e., the activation probabilities, is known to the agent and the goal is to  maximize the influence spread, i.e., total number of activated users. IM has been studied in a non-adaptive setting where the agent chooses the seed nodes before the diffusion starts \cite{kempe2003maximizing, chen2009efficient,chen2010scalable, bourigault2016representation, Netrapalli:2012:LGE:2318857.2254783,Saito:2008:PID:1430307.1430318}, or an adaptive setting where the agent sequentially selects the seed nodes adaptive to current diffusion results \cite{golovin2011adaptive, tong2016adaptive,han2018efficient,peng2019adaptive,tong2020adaptive}.
Online influence maximization \cite{chen2013combinatorial, chen2015online, lei2015online, lugosi2019online, perrault2020budgeted, zuo2022online} is proposed to learn network model while selecting seed nodes in the non-adaptive setting. Existing works on online IM studies mostly follow IC model and edge-level feedback \cite{chen2013combinatorial, wang2017improving, wen2017online, vaswani2017model, lugosi2019online}. 
\citet {chen2013combinatorial} and  \citet{wang2017improving} formulated the online IM problem as combinatorial bandits problem and proposed combinatorial upper confidence bound (CUCB) algorithm to estimate the activation probabilities of edges in a tabular manner. % 
\citet{wen2017online} assumed a linear parameterization on each edge with known edge features and proposed a linear bandits-based solution. \citet{wu2019factorization} considered the influence power and susceptibility on each node as two unknown latent parameters and proposed a matrix factorization-based solution. Our paper is the first to consider online influence maximization in the adaptive setting and formulate it as an RL problem. We also can handle the more challenging node-level feedback.
Some recent works also explored settings beyond IC model and edge-level feedback. \citet{li2020online} studied online IM with linear threshold model, and proposed a linear bandits-based solution to model the linearity in LT model for node-level feedback. We also leveraged the linearity in diffusion model to handle node-level feedback similar to \cite{li2020online} but for IC model. \citet{vaswani2017model} considered diffusion model-independent setting using a heuristic objective function, but without theoretical guarantee of the heuristic. \citet{olkhovskaya2018online} studied UCB-based algorithm for node-level feedback with theoretical guarantee, but their algorithm is designed only for certain random graph models such as stochastic block models and Chung–Lu models.

Our analysis is related to regret analysis of model-based reinforcement learning, which have been studied in various settings such as tabular MDP \cite{auer2008near}, linear/kernel MDP \cite{yang2020reinforcement,NEURIPS2020_9fa04f87}, factored MDP \cite{rosenberg2021oracle}, general model class \cite{ayoub2020model}, etc. We provide a first problem-specific analysis for influence maximization. Our analysis differs from existing regret analysis in a couple of ways. First, although we focus on a linear model for network diffusion, the state-to-state transition of the IM is highly nonlinear, thus the value and Q functions for IM do not admit a linear model and invalidate linear/kernel MDP approaches. Second, due to the nature of network diffusion process, the state and its value can grow unboundedly for large networks, causing unbounded variance at the same time. Our analysis is specially tailored to such growth process over large networks and derive regret bound by focusing a high probability event where states stay bounded. To our best knowledge, this is a first IM-specific regret analysis for controlling unbounded growth process over large networks.

%% file: sections/theory.tex
\section{Problem Formulation}

We present a tensor network diffusion process to model user feature-dependent content feature-dependent network propagation. Our goal is to both select seed users and customize contents for influence maximization. Further, we formulate IM into an RL problem to enable much more delicate control of the network diffusion process based on real-time feedback. 

\subsection{Tensor Network Diffusion Process}

Consider a social network of $N$ users, where the network structure may be hidden. Let there be $K$ choices of contents.
Let $s_{i,k}\in\{0,1\}$ denote the status of an user-content pair, i.e., $s_{i,k} = 1$ if user $i$ is actively tweeting content $k$.  
The full state of the network is denoted by $s \in \{0,1\}^{N \times K}$, a binary matrix.
We focus on the \textbf{asymptotic region of large networks}, i.e., $N$ can be arbitrarily large or even $N\to\infty$. In this regime, we have $T\ll N$, in other words, one has only a little time to learn about a huge or even infinite network.  

We assume that each content can be propagated from one user to multiple users following an independent network diffusion process. 
\begin{assumption}[Bernoulli Independent Cascade Model\label{bernoulli}] 
Let $s'$ be the next state. For each $k\in[K]$, we assume there is an underlying connectivity matrix $\Ab^k \in \mathbb{R}^{n \times n}$ such that
\begin{equation}
    \mathbb{P}(s_{i,k}' = 1|s) = \sum_{j} \Ab^k_{i,j} s_{j,k},
    \label{eq:user-content-transition}
\end{equation}And we assume $s_{i,k}'$'s are independent conditioned on $s$.
\end{assumption}
Here $\Ab^k_{i,j}$ measures the level of influence user $j$ has over user $i$ for the $k$-th content. Therefore, the aggregate ``influence'' received by user $i$ is $\sum_{j} A^k_{i,j} s_{j,k}$. We model the status of user $i$ as a Bernoulli variable, which is parameterized by the aggregate ``influence'' received by user $i$.

Our model is closely related to the independent cascade model \cite{kempe2003maximizing}. In IC model, the activation probability takes of the form
$  \mathbb{P}(s_{i,k}' = 1|s) =  1-\prod_{j}(1- \Ab^k_{i,j} s_{j,k}).$ A limitation is that efficient estimation of IC model requires edge-level observations \cite{chen2013combinatorial, wang2017improving}.
Assumption \ref{bernoulli} can be viewed as an linearized approximation to IC model, i.e., $   1-\prod_{j}(1- \Ab^k_{i,j} s_{j,k}) \approx \sum_{j} \Ab^k_{i,j} s_{j,k}$ when all the $A$ values are tiny.
This linearization gives an $\gO(1/N)$ gap measured in total variation-divergence, which would yield a $\gO(T/N)$ gap in the regret. Since $T\ll N$, we consider the gap negligible in the rest of the paper.  
%
%
%
% generalized Linear setting here.
In Appendix~\ref{sec:generalized}, we extend the assumption to the generalized linear setting and establish the regret bound for our algorithm.

Consider a parameterized network diffusion model based on user features and content features. Let the $i$-th user be associated with a user feature vector $x_i \in \mathbb{R}^{d_1}$, for all $i\in[N]$. Let the $k$-th content be associated with a content feature  $\theta_k \in \mathbb{R}^{d_2}$, for all $k\in[K]$. We assume that the influence is linear with respect to both user and content feature. 
\begin{assumption}[Linear Tensor Model]\label{assump:2}
There exists a $d_1\times d_1\times d_2$ tensor $\gT^* \in \mathbb{R}^{d_1 d_1 d_2} $ such that
\[
    \Ab _{i,j}^k = \langle \gT^*,  x_i \otimes x_j \otimes \theta_k \rangle,
\]
where $\otimes$ denotes outer product and $\langle,\rangle$ denotes inner product.
\end{assumption} 

Note that this is different from the linear MDP model commonly studied in the theoretical RL literature \cite{jin2020provably}. We focus on large networks where $N$ can be arbitrarily large. We also assume each individual user has bounded influence over its neighbors and the diffusion process has a natural decay property.  

\begin{assumption}[Uniform transition probability upper bound]\label{assump:utpub}
There exists a constant $C>0$ such that 
$
    \|\Ab\|_{\infty}\leq\frac{C}{NK}.
$
\end{assumption}

\begin{assumption}[Diffusion decay\label{decay}] 
There exists $\Delta>0$ such that $\sum_{i} \Ab^k_{i,j} \leq 1 - \Delta $ for all $k,j$.
\end{assumption}

Assumption \ref{decay} says that influence from any seed user has a discounting nature; without this assumption, some seed user may have infinite-long influence and make the diffusion process unbounded. This assumption also implies that, the ``influence" of any seed user-content pair would last $\gO(1/\Delta)$ time steps.

\subsection{Reinforcement Learning Model}

We formulate the influence maximization problem as an infinite-horizon discounted MDP. Define the state space as $\gS = \{0,1\}^{N \times K}$ where $1$ refers to activated user-content pair. %
At each timestep, the agent observes the current network state $s$ and picks an action $a\in\mathcal{A}:=[N]\times[K]$ to activate one user-content pair. Let $s_a$ be the post-action state, i.e., $s_a=s+\mathbf{1}_a$. Then the state of network transitions following the network diffusion process, i.e., Assumptions~\ref{bernoulli},\ref{assump:2}. 
\begin{comment}The next state is given by Equation~(\ref{eq:user-content-transition}) with $s$ being replaced by $s_a$. That is 
\[
    \PP(s'|s,a) = \PP(s'|s_a), \quad  (s_a)_{i,k} = \left\{
        \begin{array}{ll}
            s_{i,k}, & (i,k) \ne a, \\
            1, & (i,k) = a.
        \end{array}
    \right.
\]
\end{comment}
Since users are activated independently from one another, the state-transition law of the MDP admits a factored structure: 
$$
    \mathbb{P}(s'|s,a) =  \prod_{i,k} \mathbb{P}(s_{i,k}'|s,a).
$$ 
At each state-action pair, the agent receives a reward $r(s,a) = \sum_{i,k} v_{i,k}s_{i,k}$ measuring the amount of influence over the network. For examples, if we let $v_{i,k} \equiv 1$, then we have $r(s,a)=\|s\|_1$, which counts the number of active users. Without lost of generality, we assume $v_{i,k} \leq 1$. 
\begin{comment}\kaixuan{Need to assume the reward is deterministic and bounded between 0 and 1 somewhere. Here $r(s,a) = \|s\|_1/NK$ --- needs to be normalized. }
\mw{no. it's not fair to assume bounded $r$ in network problems.  you need to prove that $\|s\|_1/NK<1$ whp. It can't be automatically assumed. Your result is much stronger for nonnormalized $r$}
\end{comment}
Let $\pi:\mathcal{S}\mapsto\mathcal{A}$ be a decision policy. We measure the {\it value of policy $\pi$ at state $s$} as a cumulative sum of discounted rewards
\[
    V^\pi(s) = \EE^\pi \left[\sum_{t=1}^\infty \gamma^{t-1} r(s_{t}, a_{t}) \Big | s_1 = s \right].
\]
Recall Assumption~\ref{decay} that the influence of any action lasts $1/\Delta$ time steps. Thus, a natural choice of the discount factor to be $\gamma = 1- o(\Delta)$.  
Finally, the policy optimization problem is to find $\pi^*=\argmax_{\pi} V^{\pi}(s)$.
\paragraph{Relation between Discounted MDP and Bandit IM model.}
The discounted MDP formulation differs from the bandit IM optimization in two ways. (1) Our policy is dynamic and makes state-dependent decisions, while the bandit approach would make a batch of decisions only at the beginning of the diffusion process;
(2) In both cases, the optimization objectives are sums of total influences from all seed users. The difference lies in how to measure the per-seed influence. In IM bandit, the per-seed influence is a cumulative sum calculated after the diffusion process is over. In our formulation, the per-seed influence is a cumulative $\gamma$-discounted sum of rewards from this seed's descendants.
If we choose $\gamma=1 - o(\Delta)$, these values differ by only $o(1)$ and we can make the difference arbitrarily small.

%% file: sections/algorithm.tex
\section{Algorithm} \label{sec:algorithm}
 
To reduce the statistical complexity, we adopt a model-based RL approach for exploring the unknown network and learning the optimal policy. Our approach alternates between model estimation and policy update. Our algorithm calculates a bonus function based on the collected data and and add it to the reward, which dynamically trades-off between exploitation and exploration. We also adopt a slow switching technique to reduce computational burden.
 
\textbf{Tensor ridge regression for model estimate.}
 
Under the linear tensor model (Assumption~\ref{assump:2}), we can use tensor ridge regression to perform model-based RL. This reduces the statistical complexity since the dimension of the unknown parameter is smaller. Furthermore, this approach only requires node-level feedback, While previous bandits approaches for IC model require edge-level feedback \cite{chen2013combinatorial, wen2017online, wu2019factorization}.

Specifically, let $s_{a}$ be the altered state after applying action $a$. 
Observe that, conditioned on $(s,a)$, the random variable $s'_{i,k}$ satisfies a linear relation:
\[ \textstyle
    \E[s'_{i,k} |s,a] = \sum_{j} \Ab^k_{i,j} (s_{a})_{j,k}=  \left\langle \gT^*, x_i \otimes \left(\sum_{j} x_j \cdot (s_{a})_{j,k} \right)  \otimes \theta_k \right\rangle .
\]
Denote for short   
$
    \phi_{i,k} (s,a) = x_i \otimes \left(\sum_{j} x_j \cdot (s_{a})_{j,k} \right) \otimes \theta_k \in \RR^{d_1 d_1 d_2},
$
and $\phi^t_{i,k} = \phi_{i,k} (s_t,a_t)$. At time $t$, after observing the history $(s_1, a_1,\dots, s_{t-1}, a_{t-1}, s_t)$, we estimate the tensor model by :
\[
   \widehat {\gT}_t = \argmin_{\gT} \sum_{\tau=1}^{t-1} \sum_{k=1}^K \sum_{i=1}^N (\langle \gT, \phi_{i,k}^\tau \rangle - (s_{\tau+1})_{i,k})^2 + \lambda \| \gT \|_2^2,
\]
where $\| \gT \|_2^2$ is calculated by vectorizing $\gT$.
This allows an analytical solution:
\begin{equation}
    \widehat{\gT}_{t} = \Sigma_{t-1}^{-1} B_{t-1},  \label{eq:That}
\end{equation}
where
\begin{equation}
    \Sigma_{t-1} = \lambda \boldsymbol I + \sum_{\tau=1}^{t-1} \sum_{k=1}^K \sum_{i=1}^N \phi^\tau_{i,k}\cdot (\phi^\tau_{i,k})^\top.  
\qquad
    B_{t-1} =  \sum_{\tau=1}^{t-1} \sum_{k=1}^K \sum_{i=1}^N \phi^\tau_{i,k} \cdot (s_{\tau+1})_{i,k}. \label{eq:Sigma}
\end{equation}

\textbf{Optimistic Planning with truncated-reward model.}
To avoid the worst-case $O(NK)$ reward, we identify a high probability upper bound $\Lambda$ for the rewards and truncate the reward as $\widetilde r (s,a) = \min\{r(s,a), \Lambda\}$.
Then based on the ridge regression estimation $\widehat{\gT}_t$, we add a bonus term $b_t(s,a)$ to the truncated reward $\widetilde r$ and solve for an optimistic Q-function $Q^*_{\widehat{\gT}_t, \widetilde r+b_t}(s,a)$ using the model estimate. 

Specifically, we can choose $\Lambda = \frac{6}{\Delta^2} \log(4NKT^3)$. For $\widehat{\gT}_t$, we define the reward bonus as  
\begin{equation}
    b_t(s,a) = \frac{2 \gamma \Lambda}{1-\gamma} \sum_{i=1}^N \sum_{k=1}^K (1 \wedge \beta_t \cdot \| \phi_{i,k} (s,a) \|_{\Sigma_{t-1}^{-1}} ) ,
    \label{eq:bonus}
\end{equation}
where we use the notation $1 \wedge x = \min\{1, x\}$ and   
\begin{align}
    \beta_t =&   \Big(\frac{24}{\Delta} \sqrt{ {C_A}/{(NK)} \cdot d \cdot  \log(1+NK L^2 t/(d\lambda))} + 4 \Big) \log(8N^2K^2t^2/\delta) +  \sqrt{\lambda}\|\gT^*\|_2 \label{eq:beta:t}
\end{align}
with $L$ being an upper bound of $\|\phi_{i,k}^t\|_2$  and $d = d_1^2d_2$.

This choice of $\beta_t$ ensures with high probability, $Q^*_{\widehat{\gT}_t, \widetilde r+b_t}(s,a)$ is an upper bound of $\widetilde Q^*(s,a)$, which is the optimal Q-function for ground-truth transition with truncated-reward. 
We calculate the optimal truncated Q-function $Q^*_{\widehat{\gT}_t, \widetilde r+b_t}(s,a)$ using value iteration with truncation (Algorithm~\ref{algo:vit}).
\paragraph{Slow switching.} To reduce computation overhead, we adopt a slow switching technique from bandit and RL literatures \citep{abbasi2011improved, zhou2021provably}. The idea is that we only update model and policy when enough new data has been collected, via checking the covariance matrix. Specifically, say the most recent switching happens at time $t$, we choose to switch at time $t'$ only if $$ \det(\Sigma_{t'-1}) > 2 \det(\Sigma_{{t}-1}).
$$
After switching, we calculate the optimistic Q-function $Q_{t'} = Q^*_{\widehat{\gT}_{t'}, \widetilde r+b_{t'}}(s,a)$. Then we pick actions greedily using $Q_t'$, i.e.,
$ a = \argmax_{a} Q_{t'}(s,a)$, until the next switching.
 
\begin{algorithm}[h]
    \caption{Model-based RL for Influence Maximization (\textsc{Morima})} \label{algo:1}
    \begin{algorithmic}[1]  
        \STATE Initialize $\Sigma_1 = \lambda \boldsymbol I $, $B_1 = \boldsymbol 0$. $Z = \det(\Sigma_1)$.
        \STATE  Calculate $\widehat{\gT}_1$ and $b_1(s,a)$ and  compute $Q_1 = Q^*_{\widehat{\gT}_1, \widetilde r+b_1}(s,a)$.
        \STATE Take the greedy action with respect to $Q_1$: $a_1 = \argmax_a Q_1(s_1,a)$.
        \FOR{$t=2, \cdots, $}
            \STATE Calculate $\Sigma_{t-1}$ and $B_{t-1}$ according to Eqn.~\eqref{eq:Sigma}.
            \IF{$\det(\Sigma_{t-1}) > 2Z $}
                \STATE Calculate $\widehat{\gT}_t$ and $b_t(s,a)$ according to Eqn.~\eqref{eq:That} and Eqn.~\eqref{eq:bonus}. 
                \STATE Compute the optimistic Q-function $Q_t = Q^*_{\widehat{\gT}_t, \widetilde r+b_t}(s,a)$ (Algorithm~\ref{algo:vit}). 
                \STATE Set $Z = \det(\Sigma_{t-1})$.
            \ELSE 
                \STATE Set $Q_t = Q_{t-1}$.
            \ENDIF
            \STATE Take the greedy action with respect to $Q_t$: $a_t = \argmax_a Q_t(s_t,a)$.
        \ENDFOR
    \end{algorithmic}
\end{algorithm}

\begin{algorithm}[h]
    \caption{Truncated Value Iteration} \label{algo:vit}
    \begin{algorithmic}[1]
    \STATE \textbf{Input:} parameter  $\gT$, reward $\widetilde r(s,a)$, bonus term $b(s,a)$.
        \STATE Initialize $Q(s,a) = \frac{\Lambda}{1-\gamma}$.
        \WHILE{Not Converged}
             \item $Q(s,a) \gets \min \{\frac{\Lambda}{1-\gamma},\widetilde r(s,a) + b(s,a) + \gamma \EE_{s'\sim \PP_\gT(\cdot |s,a)} \max_{a'} Q(s',a') \}  $
        \ENDWHILE
    \STATE \textbf{Return}: Optimistic Q function $Q(s,a)$.
    \end{algorithmic}
\end{algorithm}

\paragraph{Full algorithm.}
We put together the pieces and present the full Algorithm~\ref{algo:1}. 
The algorithm makes only $\gO(d\log(T))$ model updates and policy updates until time $T$. Each model update can be done efficiently using least square regression. Policy updates require  solving a new planning problem which can be combinatorially hard. In practice, one can solve the planning problem using Monte-Carlo Tree Search (MCTS) methods \cite{browne2012survey}, which work well as an approximated planner in our experiments. For theoretical analysis, we assume access to a planning oracle that is able to find the optimal policy with respect to a known model $\widehat{\gT}$. Relaxing such assumption to an approximated planning oracle can be also be done with minor algorithmic and analysis modifications.

%% file: sections/reg.tex
\section{Regret Analysis}

In this section, we provide regret analysis for Algorithm~\ref{algo:1}. We define the regret for the infinite-horizon discounted MDP as in \citep{zhou2021provably}.  %\kaixuan{Maybe several words on why we define the regret in this way.} 
\begin{definition} For any possibly non-stationary policy $\pi$, the infinite-horizon discounted regret is defined as 
\[
    \reg (T) = \sum_{t=1}^T \Delta_t, \text{ where } \Delta_t = V^*(s_t) - V_t^\pi(s_t),
\]
where $V^*$ is the optimal value function, and $V_t^\pi$ is defined as 
\[
    V_t^\pi(s) = \EE^\pi \Big[\sum_{i=0}^\infty \gamma^i r(s_{t+i}, a_{t+i}) | s_1, \dots, s_{t-1}, s_t = s \Big]
\]
\end{definition}

Now we present our main theorem.  

\begin{theorem}\label{thm:main}
Let Assumptions~\ref{bernoulli}-\ref{decay} hold. With probability at least $1-\delta$, Algorithm~\ref{algo:1} satisfies the following regret upper bound:
\[
        \reg (T) \leq \widetilde{\gO} \Big( \frac{1}{\Delta ^2(1-\gamma)^2} \cdot \big(d\sqrt{C }/\Delta + \sqrt{dNK} \big) \cdot \sqrt{T}  \Big) + \mathrm{polylog}(T)\text{-}\mathrm{terms}, 
\]
where $d = \mathrm{dim} (\gT^*) = d_1^2d_2$.
\end{theorem}

We see that the dominant term of the regret is $\widetilde{\gO} \Big( \frac{1}{\Delta ^2(1-\gamma)^2} \cdot \big(d\sqrt{C }/\Delta + \sqrt{dNK} \big) \cdot \sqrt{T}  \Big)$. Notice that the worst-case reward would scale with $NK$, while we managed to reduce the scaling of the regret to $1/\Delta^2$. 

Next, we provide a proof sketch and defer the complete proof to Appendix~\ref{append:proof}.

\begin{proof}[Proof sketch of Theorem~\ref{thm:main}]

We highlight several key components of the proof.

\textbf{High probability upper bounds for the size of active user-content pairs.} We utilize the diffusion decay assumption (Assumption~\ref{decay}) to provide a high probability upper bound on the number of active user-content pairs. We show that for any policy $\pi$, with probability at least $1-p$, we have for all $t\geq 1$,
\begin{equation}
    \|s_t\|_1 \leq   {\gO}(\frac{1}{\Delta^2}\log\frac{4t^2}{p}). \label{eq:uc-bound}
\end{equation}
We see that although we have in total $NK$ user-content pairs, the number of active ones is constrained by a constant intrinsic to the network diffusion dynamics.

\textbf{Sharper bounds for the confidence region.} We derive a batched version of Bernstein-type self-normalized bounds from \citep{zhou2021provably} and show that with high probability for all $t$, $\| \widehat{\gT}_t - \gT^* \|_{\Sigma_{t-1}} \leq \beta_t$, where $\beta_t$ can be chosen as $\widetilde{\gO}(\sigma\sqrt{d} + 1)$ and $\sigma^2$ is the upper bound of $\mathrm{var}[(s_{t+1})_{i,k}|s_t,a_t]$. Combing Eqn.~\eqref{eq:uc-bound} and Assumption~\ref{assump:utpub}, we have 
\begin{align*}
    \mathrm{var}[(s_{t+1})_{i,k}|s_t,a_t] \leq \EE[(s_{t+1})_{i,k}|s_t,a_t] = \sum_j \Ab^k_{i,j}(s_{ta_t})_{j,k}  \leq \frac{C}{NK}(\|s_t\|_1 + 1) \leq \widetilde{\gO}(\frac{C}{NK\Delta^2}).
\end{align*}
Then $\beta_t = \widetilde{\gO}(\sqrt{\frac{dC}{NK}}/\Delta+ 1)$, which improves upon $\beta_t = \widetilde{\gO}(\sqrt{d})$ given by the sub-Gaussian type self-normalized bounds.

\textbf{Surrogate regret of the truncated-reward model.} Since we essentially run our algorithm against the truncated-reward model, we define the surrogate regret as $\widetilde{\reg}(T) = \sum_{t=1}^T(\widetilde{V}^*(s_t) - \widetilde{V}^\pi_t(s_t))$, where $\widetilde{V}^*$ and $\widetilde{V}^\pi_t$ are computed using the truncated reward $\widetilde{r}(s,a) = \min\{r(s,a), \Lambda\}$. By Eqn.~\eqref{eq:uc-bound}, with probability at least $1-1/(2NKT)$, under any policy, we have  $r(s_t,a_t) \leq \|s_t\|_1 \leq \Lambda = \widetilde{\gO}(1/\Delta^2)$ for all $t\leq T$. This means with high probability we have $r(s_t,a_t) = \widetilde r(s_t,a_t)$ and hence the true regret and the surrogate regret is similar. Specifically, we will show  $\reg(T) \leq \widetilde{\reg}(T) + 1/(1-\gamma)$.

\textbf{Bonus term.} Let $\PP(s'|s,a)$ be the true transition probability and $\widehat \PP (s'|s,a)$ be the empirical estimate. As a typical result in MDP theory, we require  $b(s,a) \geq \gamma \|V\|_\infty \cdot \| \PP(\cdot|s,a) - \widehat \PP(\cdot|s,a) \|_1$ to ensure optimism. We exploit the fact that $\PP(s'|s,a)$ and $\widehat \PP(\cdot|s,a) $ is factorized, i.e.,  $\displaystyle \PP(\cdot|s,a) = \otimes _{i=1}^N\otimes _{k=1}^K \PP_{i,k}(\cdot|s,a)$ and $\widehat \PP(\cdot|s,a) = \otimes _{i=1}^N\otimes _{k=1}^K \widehat\PP_{i,k}(\cdot|s,a)$, which stem from the independence assumption (Assumption~\ref{bernoulli}). This gives us 
$
    \| \PP(\cdot|s,a) - \widehat \PP(\cdot|s,a) \|_1 \leq  \sum _{i=1}^N \sum_{k=1}^K \| \PP_{i,k}(\cdot|s,a) - \widehat \PP_{i,k}(\cdot|s,a) \|_1.
$
Notice that $\PP_{i,k}(\cdot|s,a)$ is a Bernoulli distribution, then by Assumption~\ref{assump:2} we have $ \| \PP_{i,k}(\cdot |s,a) - \widehat\PP_{i,k}(\cdot |s,a) \|_1 \leq 2(1 \wedge |\langle \gT^* - \widehat{\gT} , \phi_{i,k}(s,a) \rangle |) $. Therefore, the bonus term can be chosen as Eqn.~\eqref{eq:bonus} and we ensure optimism at each time.

\textbf{Regret decomposition.} We have the following regret decomposition for the surrogate regret.
\begin{align*}
    \widetilde{\reg}(T) \leq  \gO\Big\{\frac{1}{1-\gamma}\Big[  \sum_{t=1}^T b_{t_s}(s_t,a_t) +  \frac{2\gamma\Lambda}{1-\gamma} \sqrt{T\log\frac{1}{\delta}} + \Big( \frac{\Lambda}{1-\gamma} M \Big) \Big] \Big\},
\end{align*}
where $M = M(T)$ is the total number of switches and we will show that $M  = \widetilde{\gO}(d)$. Then the dominant term of the regret is
\begin{align*}
    \frac{1}{1-\gamma} \sum_{t=1}^T b_{t_s}(s_t,a_t) &=  \widetilde{\gO}\Big (\frac{1}{(1-\gamma)^2\Delta^2} \Big) \cdot \beta_T \cdot \sum_{t=1}^T \sum_{i=1}^N \sum_{k=1}^K (1 \wedge  \| \phi_{i,k} (s_t,a_t) \|_{\Sigma_{{t_s}-1}^{-1}} )\\
    &\leq \widetilde{\gO}\Big (\frac{1}{(1-\gamma)^2\Delta^2} \Big) \cdot \beta_T \cdot  \widetilde{\gO}(\sqrt{dNKT} + MNK).
\end{align*}
where $t_s$ denotes the last switch up to time $t$, and the last inequality follows from a variant of Elliptical Potential Lemma. Plug in the choice of $\beta_T$ and we derive the result.

\end{proof}

%% file: sections/experiment.tex
\section{Experiments}

\label{sec-experiment}

\paragraph{Data description.}

We run benchmark experiments on synthetic data to evaluate the performance of our Algorithm \ref{algo:1}. We construct a synthetic directed graph consisting of $N=300$ nodes where each node corresponds to an user. We construct 6-dim user features for each node, $x_i \in \RR^{d_1}, d_1=6$. Next, we make content features $\theta_k \in [0,1]^{d_2}, d_2=3, ||\theta_k||_{\infty}=1$ for $K=4$ contents. Specifically, each orthogonal content feature direction corresponds to $100$ users, i.e. these user can be activated by only one associated content dimension. Also we have 3 types of users have different influencing power of high, mediate and low. The users with high-level influencing power lead to two-step delayed reward, while the mediate users lead to good immediate reward. Dynamic planning would be key for such a setting. Finally, We use a fixed tensor $\gT^* \in \RR^{d_1d_1d_2}$ as the underlying dynamic to induce transition matrices $\Ab^k, k \in [K]$ by Eqn.~\eqref{eq:user-content-transition}.

The state space $S$ of the reinforcement learning problem has size $|\gS|=2^{NK}$; a state of the graph represents active/inactive status of all (user, content) pairs, i.e. $s \in \{0,1\}^{N \times K}, NK=1200$. The action space consists of all $1200$ user-content pairs, i.e. an action $a \in [N] \times [K]$. At time step $t$, the state of the system will update according to Eqn.~\eqref{eq:user-content-transition}.

\medskip

\paragraph{Implementation and Baselines.}

Exactly solving for the optimal policy, even if the network is fully known, requires solving a combinatorially hard planning problem. In our experiment, we adopt a two-step lookahead scheme for approximate dynamic programming in \textsc{Morima}. The parameters in Algorithms \ref{algo:1} and \ref{algo:vit} are set as $\gamma = 0.9, \lambda = 1$.

We compare \textsc{Morima} algorithm with the following baselines.
\begin{itemize}
    \item Random policy, which uniformly selects a user-content pair to activate. 
    
    \item IMLinUCB \cite{wen2017online}. A combinatorial linear bandits baseline that was originally designed for non-adaptive online IM. By setting budget as $b=2$, we run the algorithm every $b$ rounds and play the selected $b$ actions spontaneously.
    
    \item \textsc{Morima} without slow switching. We force the Q-function to be updated at each time step.
    
    \item \textsc{Morima} with one-step lookahead. Here we take the average immediate reward based on estimated $\widehat{\Ab}$ as the value of Q-function.
    
     \item \textsc{Morima} with known $\Ab^k$s. We input ground-truth connectivity matrices $\Ab^k$s as input for two-step lookahead planning. We consider this as the performance upper bound of our algorithm with the same approximate planning oracle.

\end{itemize}

\paragraph{Results and analysis.}

We report the averaged discounted sum of rewards and its variance of total 20 runs in Figure \ref{fig-reward}. The discounted sum of reward of \textsc{Morima} reaches the same level of the performance upper bound with true $\Ab$ in less than 100 rounds, showing that our algorithm can quickly explore the unknown network and learn to make optimal decisions. We also notice that our algorithm outperforms IMLinUCB because it can adaptively make decisions based on current state while IMLinUCB makes static decisions. We can observe that deeper lookahead benefits the \textsc{Morima} performance; the reward by \textsc{Morima} with only one-step lookahead is much lower suggesting that a good approximated online planning oracle is preferred.

\begin{figure}
\centering
\includegraphics[width=.99\textwidth]{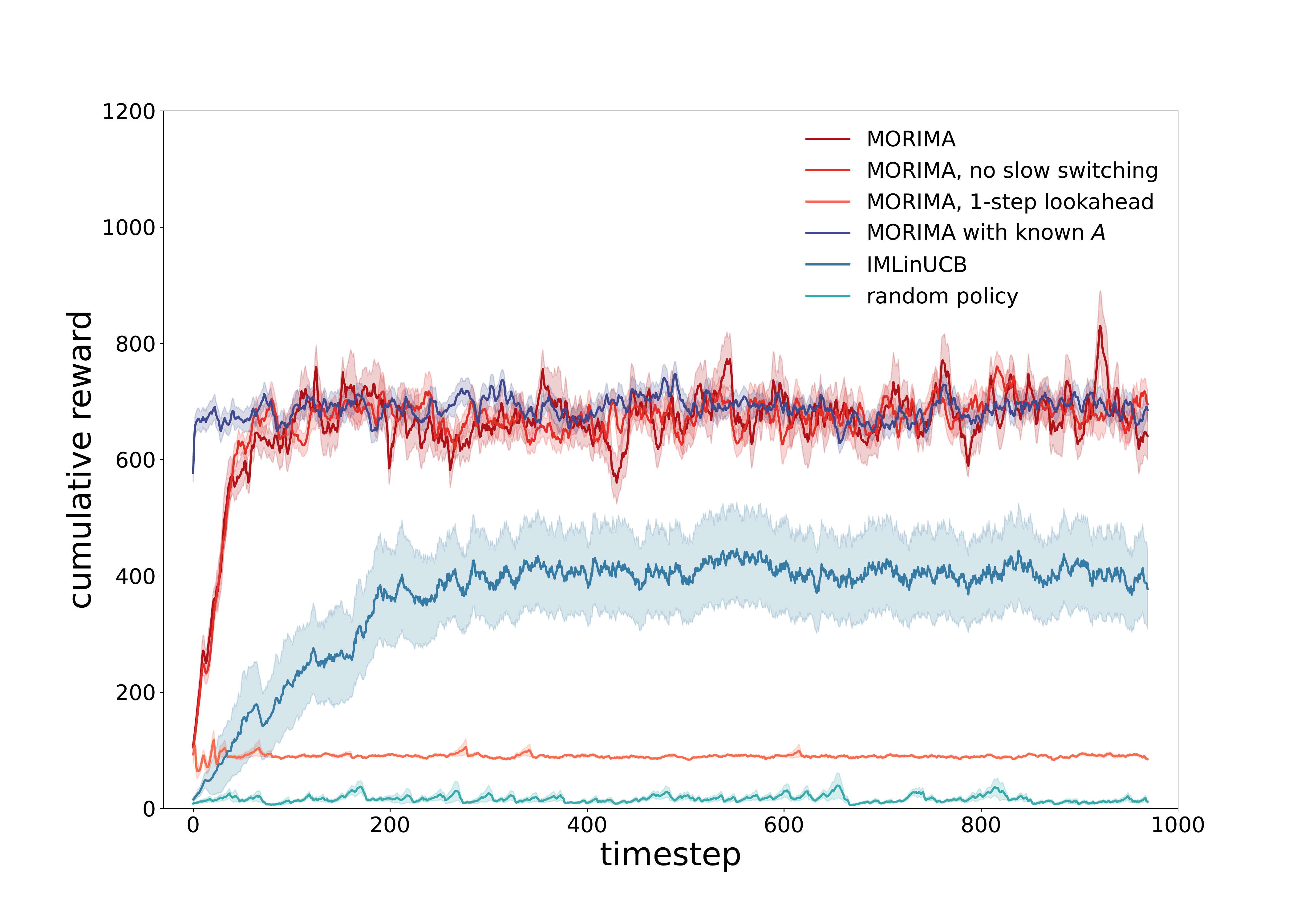}
\vspace{-2mm}
\caption{\textbf{Real-time discounted sum of rewards.} Synthetic network with static underlying dynamics $\gT^*$ generating edge-level transition probability. } 
\centering
\label{fig-reward}
\end{figure}

%% file: sections/conclusion.tex
\section{Conclusion} \label{sec:conclusion}
In this paper, we study the problem of content-dependent online adaptive influence maximization and formulate the problem as an infinite-horizon discount MDP. We propose \textsc{Morima}, a model-based reinforcement learning algorithm that learns optimal policy from node-level feedback under IC model. We provide a $\widetilde \gO( d\sqrt{T}/\Delta+ \sqrt{dNKT})$ regret bound for our algorithm, which is the first sublinear regret bound of online adaptive influence maximization problem, and empirically validated the effectiveness of our algorithm. As of the future works, it is interesting to investigate other diffusion models such as linear threshold model or diffusion-independent setting.

%% file: appendix/self-normalized.tex
\section{A Sharper bound for the Confidence Region}\label{sec:aa} 
\subsection{Main Lemma}\label{sec:B1}
\begin{lem}[Confidence Region]\label{lem:conf}
 Let Assumptions~\ref{bernoulli}-\ref{decay} hold. With probability at least $1-\delta$, we have for all $t \geq 1$, 
\[
    \| \widehat{\gT}_t - \gT^* \|_{\Sigma_{t-1}} \leq \beta_t,
\]
where  
\[
    \beta_t =  \Big(\frac{24}{\Delta} \sqrt{ {C}/{(NK)} \cdot d \cdot  \log(1+NK L^2 t/(d\lambda))} + 4 \Big) \log(8N^2K^2t^2/\delta) +  \sqrt{\lambda}\|\gT^*\|_2. 
\]
and $L = \sup \|\phi_{i,k}^t\|_2$, $d = d_1^2d_2$.
\end{lem}

Before the proof of Lemma~\ref{lem:conf}, we introduce two lemmas below:

\begin{lem}[High probability bounds for the number of active user-content pairs]\label{lem:hpbound:user}
    Let Assumptions~\ref{bernoulli}-\ref{decay} hold. For any possibly non-stationary policy $\pi$, with probability at least $1-\delta$, we have for all $t\geq 1$, 
    \[
        \|s_t\|_1 \leq \frac{2}{\Delta}(\frac{2}{\Delta}\log \frac{2 t^2}{ \delta} + 1).
    \]
\end{lem}

\begin{lem}[Bernstein-type self-normalized bound, batched version \citep{zhou2021nearly}]\label{lem:self-normed}
    Let $\{\gF_t\}_{t=1}^\infty$ be a filtration, $\{x_t^i,y_t^i\}_{t\geq 1, 1\leq i\leq m}$ be a stochastic process such that $x_t^i \in \RR^d$ is $\gF_t$-measurable and $y_t^i \in \RR$ is $\gF_{t+1}$-measurable. Assume that conditioned on $\gF_t$, $\{y_t^1,\cdots,y_t^m\}$ are independent, and 
    \[
        |y_t^i| \leq R, \quad \EE[y_t^i|\gF_t] = \langle \gT^*, x_t^i \rangle,\quad \mathrm{var}[y_t^i|\gF_t] \leq \sigma^2,\quad \|x_t^i\|_2 \leq L,
    \]
    then with probability at least $1-\delta$, the following holds simultaneously for all $t\geq 1$: 
    \[
        \| \widehat{\gT}_t - \gT^* \|_{\Sigma_{t-1}} \leq \beta_t , \quad \Big \| \sum_{i=1}^m\sum_{\tau=1}^{t-1} x_\tau^i (y_\tau^i - \langle \gT^*, x_\tau^i \rangle) \Big  \|_{\Sigma_{t-1}^{-1}} \leq \beta_t - \sqrt{\lambda}\|\gT^*\|_2.
    \]
    where $\widehat{\gT}_t = \Sigma_{t-1}^{-1} B_{t-1}$, $\Sigma_{t-1} = \lambda I + \sum_{i=1}^m \sum_{\tau=1}^{t-1} x_\tau^i (x_\tau^i)^\top$,  $B_{t-1} = \ \sum_{i=1}^m \sum_{\tau=1}^{t-1} x_\tau^i y_\tau^i$, and 
    \[
        \beta_t = 8\sigma \sqrt{d\log(1+mtL^2/(d\lambda))\log(4m^2t^2/\delta)} + 4R\log(4m^2t^2/\delta) + \sqrt{\lambda}\|\gT^*\|_2.
    \]
\end{lem}

\begin{proof}[proof of Lemma~\ref{lem:conf}]

We use Lemma~\ref{lem:self-normed} for batched stochastic process $\{\phi_{i,k}^t, (s_{t+1})_{i,k}\}$. Notice that we can choose $\sigma^2$ to be the upper bound of $\mathrm{var}[(s_{t+1})_{i,k}|s_t,a_t]$, and 
\[
    \mathrm{var}[(s_{t+1})_{i,k}|s_t,a_t] \leq \EE[(s_{t+1})_{i,k}|s_t,a_t] = \sum_j \Ab^k_{i,j}(s_{ta_t})_{j,k} \leq \frac{C}{NK} \sum_j (s_{ta_t})_{j,k} \leq \frac{C}{NK}(\|s_t\|_1 + 1).
\]
where we used the assumption that $\Ab^k_{i,j}\leq\frac{C}{NK} $. By Lemma~\ref{lem:hpbound:user}, we have with probability at least $1-\delta/2$, for all $t$,
\[
    \|s_t\|_1 \leq   \frac{2}{\Delta}(\frac{2}{\Delta}\log \frac{4 t^2}{ \delta} + 1).  
\]
Therefore, when the above inequalities hold, we have 
\[
    \mathrm{var}[(s_{t+1})_{i,k}|s_t,a_t] \leq \frac{C}{NK}(\frac{2}{\Delta}(\frac{2}{\Delta}\log \frac{4 t^2}{ \delta} + 1)+ 1) \leq  \frac{C}{NK}(\frac{3}{\Delta})^2\log \frac{4 t^2}{ \delta}.
\]
By Lemma~\ref{lem:self-normed} with $m=NK$, $R=1$, and $\sigma^2 =  \frac{C}{NK}({3}/{\Delta})^2\log \frac{4 t^2}{ \delta}$, we have 
\[
    \beta_t =   \frac{24}{\Delta} \sqrt{ {C}/{(NK)} \cdot d \cdot \log ({4 t^2} / { \delta}) \log(1+NK L^2 t/(d\lambda))\log(8N^2K^2t^2/\delta)} + 4 \log(8N^2K^2t^2/\delta) + \sqrt{\lambda}\|\gT^*\|_2,
\]
which is smaller than the result stated in the lemma.
\end{proof}

\subsection{Deferred proofs in Subsection~\ref{sec:B1}}

\begin{proof}[proof of Lemma~\ref{lem:hpbound:user}]
    First, we bound the expectation of $\|s_t\|_1$. By the transition, we have 
    \[
        \EE [(s_{t+1})_{i,k}|s_t,a_t]  = \sum_j \Ab^k_{i,j} (s_{ta_t})_{j,k} 
    \]
    Therefore,
    \[
        \EE [\|s_{t+1}\|_1  |s_t,a_t] = \sum_{i,j,k} \Ab^k_{i,j} (s_{ta_t})_{j,k} 
    \]
    Recall that we have assumed that for any content $k$ and any user $j$, 
    \[
        \sum_{i} \Ab^k_{i,j} \leq 1 - \Delta. 
    \]
    Then we have  
    \begin{equation}
        \EE [\|s_{t+1}\|_1  |s_t,a_t] \leq (1 - \Delta) \sum_{j,k} (s_{ta_t})_{j,k}  = (1-\Delta) \|s_{ta_t}\|_1 \leq  (1-\Delta) (\|s_{t}\|_1 + 1), \label{eq:s:1}
    \end{equation}
    where the last inequality holds since the action alters at most one entry of the state.
    
    Next, notice that conditioned on $(s_t,a_t)$, $\|s_{t+1}\|$ is the summation of $NK$ independent Bernoulli random variables. By Bernstein inequality, we have with probability at least $1-\delta_t$,
    \[
        \|s_{t+1}\|_1 -  \EE [\|s_{t+1}\|_1  |s_t,a_t] \leq 2(\sqrt{\sum_{i,k} \mathrm{var}[ (s_{t+1})_{i,k} |s_t,a_t] \log\frac{1}{\delta_t} } + \log \frac{1}{\delta_t}).
    \]
    Since the variance of a  Bernoulli random variable is bounded by its expectation, we have 
    \[
        \|s_{t+1}\|_1 -  \EE [\|s_{t+1}\|_1  |s_t,a_t] \leq 2( \sqrt{ \EE [\|s_{t+1}\|_1  |s_t,a_t] \log\frac{1}{\delta_t}}  + \log \frac{1}{\delta_t}),
    \]
    Therefore, by Equation~\eqref{eq:s:1}, we have 
    \begin{align*}
        \|s_{t+1}\|_1 & \leq (1-\Delta) (\|s_{t}\|_1 + 1) + 2\sqrt{(1-\Delta) (\|s_{t}\|_1 + 1) \log\frac{1}{\delta_t}} + 2 \log\frac{1}{\delta_t} \\
        & \leq (1-\Delta) (\|s_{t}\|_1 + 1) +  \frac{\Delta}{2} (\|s_{t}\|_1 + 1)  +  \frac{2}{\Delta} (1-\Delta)\log\frac{1}{\delta_t} + 2 \log\frac{1}{\delta_t} \\
        & =  (1-\Delta/2) (\|s_{t}\|_1 + 1) + \frac{2}{\Delta}\log\frac{1}{\delta_t}  ,
    \end{align*}
    where we used $at+b/t \geq 2\sqrt{ab}$ for the last inequality.
    
    Finally, we set $\delta_t = \frac{\delta}{2t^2}$ so that $\sum_t \delta_t \leq \delta$ and take union bound over all $t\geq 1$. By solving the recursion, we have with probability at least $1-\delta$, 
    \[
        \|s_t\|_1 \leq \frac{2}{\Delta}(\frac{2}{\Delta}\log \frac{2 t^2}{ \delta} + 1)
    \]
    for all $t\geq 1$. 
\end{proof}

\begin{proof}[proof of Lemma~\ref{lem:self-normed}]
    We consider a ``serialized'' stochastic process. Let $\gG_{t,i} = \sigma(\gF_{t}, y_t^1, \dots,y_t^{i-1})$. When $1\leq i \leq m$, we have $\gG_{t,i} \subseteq \gG_{t,i+1}$; while when $i=m+1$, we have $\gG_{t,m+1} = \sigma(\gF_{t}, y_t^1, \dots,y_t^{m}) \subseteq \gG_{t+1,1} = \gF_{t+1}$. Then we know that 
    \[
        \gG_{1,1} \subseteq \gG_{1,2} \subseteq \dots \subseteq \gG_{1,m} \subseteq \gG_{2,1} \subseteq \gG_{2,2} \subseteq \dots \subseteq \gG_{2,m} \subseteq \dots \subseteq  \gG_{t,1} \subseteq \gG_{t,2} \subseteq \dots \subseteq \gG_{t,m}  \subseteq \dots 
    \]
    is a filtration. Clearly we have $x_t^i$ is $\gG_{t,i}$-measurable and $y_t^i$ is $\gG_{t,i+1}$-measurable. By the conditional independence assumption, we also have 
    \[
        y_t^i | \gG_{t,i} \ =_d\   y_t^i | \gF_{t}, y_t^1, \dots,y_t^{i-1} \ =_d\  y_t^i | \gF_{t}.
    \]
    Therefore, by Theorem 4.1 of \citet{zhou2021nearly}, we have with probability at least $1-\delta$, for all $t\geq 1$ and $i=1,\dots, m$, 
    \[
        \| \widehat{\gT}_{t,i} - \gT^* \|_{\Sigma_{t,i}} \leq \beta_{t,i},
    \]
    and 
    \[
        \Big \| \sum_{\tau = 1}^{t-1} \sum_{j=1}^m  x_\tau^j (y_\tau^j - \langle \gT^*, x_\tau^j \rangle) + \sum_{j=1}^i x_t^j (y_t^j - \langle \gT^*, x_t^j \rangle) \Big  \|_{\Sigma_{t,i}^{-1}} \leq \beta_t - \sqrt{\lambda}\|\gT^*\|_2,
    \]
    where 
    \[
        \widehat{\gT}_{t,i} = \Sigma_{t,i}^{-1} B_{t,i}, \ \Sigma_{t,i} = \lambda I + \sum_{\tau = 1}^{t-1} \sum_{j=1}^m x_\tau^j (x_\tau^j)^\top + \sum_{j=1}^i x_t^j(x_t^j)^\top, B_{t,i} = \sum_{\tau = 1}^{t-1} \sum_{j=1}^m x_\tau^j y_\tau^j + \sum_{j=1}^i x_t^j y_t^j,
    \]
    and 
    \[
        \beta_{t,i} = 8\sigma\sqrt{d\log(1+ t_iL^2/(d\lambda))\log(4t_i^2/\delta)} + 4 R  \log(4t_i^2/\delta)  + \sqrt{\lambda}\|\gT^*\|_2, \quad t_i =m(t-1)+i.
    \]
    Then the result of Lemma~\ref{lem:self-normed} follows by setting $i=m$.
\end{proof}

%% file: appendix/regret.tex
\section{Proof of Theorem~\ref{thm:main}}\label{append:proof}

\textbf{Additional Notation.}

Let $t_1 = 1$, and for $s\geq 1$, the next switching time $t_{s+1}$ is recursively defined as 
\[
    t_{s+1} = \min\{t | \det(\Sigma_{t-1}) > 2 \det(\Sigma_{{t_s}-1})\}.
\]
Denote the set of switching times by $W=\{t_1,t_2,\dots,t_M\}$ where $M$ is the total number of switches. We have $1=t_1 < t_2 < \dots < t_M \leq T < t_{M+1}$. We slightly abuse the notation and use $t_s$ to denote the last switch up to time $t$, i.e., $t_s \leq t < t_{s+1}$. Then by slow switching we mean $Q_t = Q^*_{\widehat{\gT}_{t_s}, \widetilde r + b_{t_s}}$.

Recall the definition of the regrets 
\[
    \reg(T) = \sum_{t=1}^T (V^*(s_t) - V_t^\pi(s_t)), \quad  \widetilde \reg(T) = \sum_{t=1}^T (\widetilde V^*(s_t) - \widetilde V_t^\pi(s_t)),
\]
where $V, \reg$ are  defined with the original untruncated model and $\widetilde V, \widetilde \reg$ are  defined with the truncated-reward model.

\textbf{Key Lemmas.}

Before the proof of Theorem~\ref{thm:main}, we introduce several key lemmas.

\begin{lem}[optimism]\label{lem:bonus}
    Let Assumptions~\ref{bernoulli}-\ref{decay} hold. Set the bonus term to be 
   \[
       b_t(s,a) = \frac{2\Lambda\gamma}{1-\gamma} \sum_{i=1}^N \sum_{k=1}^K (1 \wedge \beta_t \cdot \| \phi_{i,k} (s,a) \|_{\Sigma_{t-1}^{-1}} ).
   \]
   Then with probability at least $1-\delta$, we have the optimistic condition $\widetilde{Q}^*(s,a) \leq Q_t(s,a)$ holds for all $t \geq 1$. 
   
   Furthermore, we have for any $V(s)$ such that $0\leq V(s) \leq \Lambda/(1-\gamma)$, 
   \[
       \gamma| \EE_{s'\sim \PP(s'|s,a)} V(s') - \EE_{s'\sim \PP_{\widehat{\gT}_t}(s'|s,a)} V(s') | \leq b_t(s,a).
   \]
\end{lem}

\begin{lem}[surrogate regret]\label{lem:sg}
    Let Assumptions~\ref{bernoulli}-\ref{decay} hold. Assume that $ T\log(1/\gamma)  \geq  \log(2NKT)$. For any policy $\pi$, we have the following connection of the regrets of the two MDPs.
\[
    \reg(T) \leq \widetilde{\reg}  (T) + \frac{1}{1-\gamma}.
\]
\end{lem}

\begin{lem}[regret decomposition \citep{zhou2021provably}]\label{lem:reg:decomp}
    Let Assumptions~\ref{bernoulli}-\ref{decay} hold. Assume at each time step $t$, the results of Lemma~\ref{lem:bonus} holds. Then with probability at least $1-\delta$,  we have the following regret decomposition
\begin{align*}
    \widetilde \reg (T) \leq \frac{1}{1-\gamma}\Big[ 2 \sum_{t=1}^T b_{t_s}(s_t,a_t) +  \frac{2\gamma \Lambda}{1-\gamma} \sqrt{T\log\frac{1}{\delta}} + \gamma   \Big(2\Lambda/(1-\gamma) + E_T \Big) \Big].
\end{align*}
where $E_T$ is the switching error 
\[
    E_T = \sum_{t=1}^T V_t(s_{t+1}) - V_{t+1}(s_{t+1}).
\]
\end{lem}

\begin{lem}[bounding the number of switches]\label{lem:switching:error}
    Let Assumptions~\ref{bernoulli}-\ref{decay} hold. The total number of the switches $M$ incurred by Algorithm~\ref{algo:1} is bounded as 
    \[
        M < \frac{1}{\log 2} d \log  \Big( \frac{d +  NKTL^2/\lambda }{d}\Big) + 1,
    \]
    where $L = \sup \|\phi_{i,k}^t\|_2$.
\end{lem}

Next we state the proof of Theorem~\ref{thm:main}.

\begin{proof}[proof of Theorem~\ref{thm:main}]

Combing Lemma~\ref{lem:bonus}, Lemma~\ref{lem:sg}, and Lemma~\ref{lem:reg:decomp}, we have with probability at least $1-2\delta$, when $ T\log(1/\gamma)  \geq  \log(2NKT)$,
\[
    \reg(T) \leq \frac{1}{1-\gamma}\Big[ 2 \sum_{t=1}^T b_{t_s}(s_t,a_t) +  \frac{2\gamma \Lambda}{1-\gamma} \sqrt{T\log\frac{1}{\delta}} + \gamma   \Big(2\Lambda/(1-\gamma) + E_T \Big) \Big] + \frac{1}{1-\gamma}.
\]
Next we provide an upper bound for $\sum_{t=1}^T b_{t_s}(s_t,a_t)$. By Lemma~\ref{lem:bonus} we know that 
\[
    b_{t_s}(s_t,a_t) \leq \frac{2\gamma\Lambda}{1-\gamma}\beta_T \sum_{i=1}^N \sum_{k=1}^K (1 \wedge  \| \phi_{i,k} (s_t,a_t) \|_{\Sigma_{t_{s}-1}^{-1}} ).
\]
For any $(i,k)$, define $\Sigma_{t,i,k} = \Sigma_{t-1} + \sum_{j=1}^{i-1}\sum_{l=1}^K \phi_{j,l}^t(\phi_{j,l}^t)^\top +  \sum_{l=1}^k \phi_{i,l}^t(\phi_{i,l}^t)^\top $. By the definition that $t_{s+1} = \min \{t| \det(\Sigma_{t-1}) > 2 \det(\Sigma_{t_s  - 1})\}$, we have $\det(\Sigma_{t_{s+1}-2}) \leq 2\det(\Sigma_{t_s-1})$. Therefore, when $t_s \leq t < t_{s+1} - 1$,  we have 
\[
    \det(\Sigma_{t,i,k}) \leq \det(\Sigma_{t}) \leq  \det(\Sigma_{t_{s+1}-2}) \leq 2\det(\Sigma_{t_s-1}).
\]
By Lemma~\ref{lem:det}, this implies
\[
    \| \phi_{i,k}(s_t,a_t) \|_{\Sigma_{t_s - 1}^{-1}}^2 \leq 2 \| \phi_{i,k}(s_t,a_t) \|_{\Sigma_{t,i,k-1}^{-1}}^2.
\]
Then we have 
\begin{align*}
        &\quad \sum_{t=1}^T \sum_{i=1}^N \sum_{k=1}^K (1 \wedge  \| \phi_{i,k} (s_t,a_t) \|_{\Sigma_{t_{s}-1}^{-1}} ) \\
        & = \sum_{t +1 \in W} \sum_{i=1}^N \sum_{k=1}^K (1 \wedge  \| \phi_{i,k} (s_t,a_t) \|_{\Sigma_{t_{s}-1}^{-1}} ) +  \sum_{t +1 \notin W} \sum_{i=1}^N \sum_{k=1}^K (1 \wedge  \| \phi_{i,k} (s_t,a_t) \|_{\Sigma_{t_{s}-1}^{-1}} )\\
        & \leq MNK +  \sqrt{NKT \sum_{t +1 \notin W} \sum_{i=1}^N \sum_{k=1}^K (1 \wedge  \| \phi_{i,k} (s_t,a_t) \|^2_{\Sigma_{t_{s}-1}^{-1}} )} \\
        & \leq MNK+  \sqrt{2NKT \sum_{t +1 \notin W} \sum_{i=1}^N \sum_{k=1}^K (1 \wedge  \| \phi_{i,k} (s_t,a_t) \|^2_{\Sigma_{t,i,k-1}^{-1}} )} \\
        & \leq MNK+  \sqrt{2NKT \sum_{t=1}^T \sum_{i=1}^N \sum_{k=1}^K (1 \wedge  \| \phi_{i,k} (s_t,a_t) \|^2_{\Sigma_{t,i,k-1}^{-1}} )} \\
        & \leq MNK+  \sqrt{2NKT \cdot 2d\log\frac{d\lambda+NKTL^2}{d\lambda} },
\end{align*}
where the last inequality follows from Lemma~\ref{lem:epl}. This implies
\[
    \sum_{t=1}^T b_{t_s}(s_t,a_t) \leq \frac{2\gamma\Lambda}{1-\gamma}  \beta_T \Big( MNK+ \sqrt{4NKdT \log\frac{d\lambda+NKTL^2} {d\lambda} } \Big).
\]

Next we bound the switching error $E_T$. Since there are in total $M$ switches, we know that there are at most $M$ non-zero terms in the summation of $E_T$. Then we have 
\[
    E_T = \sum_{t=1}^T(V_t(s_{t+1}) - V_{t+1}(s_{t+1})) \leq  \frac{\Lambda}{1-\gamma} M.
\]
Plugging the result of Lemma~\ref{lem:switching:error}, we have $ E_T = \widetilde{\gO}( \frac{\Lambda}{1-\gamma} d)$.

Therefore, we have the final regret upper bound when $T\log(1/\gamma)  \geq  \log(2NKT)$:
\[
        \reg (T) \leq \widetilde{\gO} \Big( \frac{1}{\Delta ^2(1-\gamma)^2} \cdot \big(d\sqrt{C }/\Delta + \sqrt{dNK} \big) \cdot \sqrt{T}  \Big) + \mathrm{polylog}(T)\text{-}\mathrm{terms},
\]
where $d = \mathrm{dim} (\gT^*) = d_1^2d_2$. When $T\log(1/\gamma)  \leq  \log(2NKT)$, the above inequality trivially holds.  
\end{proof}

\subsection{Deferred proofs}

\begin{proof}[proof of Lemma~\ref{lem:bonus}]
    For simplicity, define $\widehat\PP =  \PP_{\widehat{\gT}_t}$, which is the estimated transition distribution obtained using $\widehat{\gT}_t$. Notice that $\PP_{i,k}(\cdot |s,a)$ is a Bernoulli distribution with success probability $\langle \gT^* , \phi_{i,k}(s,a) \rangle $, while $\PP_{i,k}(\cdot |s,a)$ is a Bernoulli distribution with success probability $\langle \widehat{\gT}_t , \phi_{i,k}(s,a) \rangle $. Therefore, we have 
    \[
        \| \PP_{i,k}(\cdot |s,a) - 
        \widehat\PP_{i,k}(\cdot |s,a) \|_1 \leq 2(1 \wedge |\langle \gT^* - \widehat{\gT}_t , \phi_{i,k}(s,a) \rangle |).
    \]
    By Lemma~\ref{lem:conf}, with probability at least $1-\delta$, we have $\|\gT^* - \widehat{\gT}_t \|_{\Sigma_{t-1}} \leq \beta_t$ for all $t\geq 1$. Then by Cauchy Inequality, the above term can be further bounded as 
    \[
        \| \PP_{i,k}(\cdot |s,a) - 
        \widehat\PP_{i,k}(\cdot |s,a) \|_1 \leq 2(1 \wedge \beta_t \cdot \|\phi_{i,k}(s,a)\|_{\Sigma_{t-1}^{-1}}).
    \]

    By Lemma~\ref{lem:factor}, we have 
    \[
        \| \PP(\cdot |s,a) - 
        \widehat\PP(\cdot |s,a) \|_1 \leq \sum_{i,k} \| \PP_{i,k}(\cdot |s,a) - 
        \widehat\PP_{i,k}(\cdot |s,a) \|_1 \leq 2 \sum_{i,k}(1 \wedge \beta_t \cdot \|\phi_{i,k}(s,a)\|_{\Sigma_{t-1}^{-1}}).
    \]
    Then by Lemma~\ref{lem:optimism0}, we know the desired result holds.
\end{proof}

\begin{proof}[proof of Lemma~\ref{lem:sg}]
Define $\pi^*$ as the optimal policy under the original model. Then we have  
\begin{align*}
    V^*(s_t)  - \widetilde V^*(s_t) & =  V^*(s_t)  - \widetilde V^{\pi^*}(s_t) + \widetilde V^{\pi^*}(s_t) -  \widetilde V^*(s_t)  \\
    &\leq V^*(s_t)  - \widetilde V^{\pi^*}(s_t) \\
    & = \EE^{\pi^*} \Big [\sum_{i=1}^\infty\gamma^{i-1} r (s_i,a_i) \Big | s_1 =s_t \Big] -  \EE^{\pi^*} \Big[ \sum_{i=1}^\infty \gamma^{i-1} \widetilde r (s_i,a_i)  \Big| s_1 =s_t\Big],
\end{align*}
where the inequality holds because $\widetilde V^*(s_t)$ is the optimal V-function with respect to the truncated-reward model.  

Notice that by Lemma~\ref{lem:hpbound:user} with $\delta_0 = 1/(2NKT)$, we know that for policy $\pi^*$, with probability at least $1-\delta_0$, we have $\|s_i\|_1  \leq \Lambda = \frac{6}{\Delta^2} \log(4NKT^3)$ for all $1 \leq i \leq T$. Therefore, with probability at least $1-\delta_0$,  $r(s_i,a_i) = \widetilde r (s_i,a_i)$ for all $1 \leq i \leq T$. Then 
\begin{align*}
    V^*(s_t)  - \widetilde V^*(s_t) &  = \EE^{\pi^*} \Big [\sum_{i=1}^\infty\gamma^{i-1} (r (s_i,a_i) -\widetilde r (s_i,a_i))  \Big| s_1 =s_t\Big] \\
    &= \EE^{\pi^*} \Big [\sum_{i=1}^T \gamma^{i-1} (r (s_i,a_i) -\widetilde r (s_i,a_i))  \Big| s_1 =s_t\Big] + \EE^{\pi^*} \Big [\sum_{i=T+1}^\infty \gamma^{i-1} (r (s_i,a_i) -\widetilde r (s_i,a_i))  \Big| s_1 =s_t\Big] \\
    &\leq  (1-\delta_0)\cdot 0 + \delta_0 \sum_{i=1}^T \gamma^{i-1} NK +  NK \frac{\gamma^{T}}{1-\gamma} \\
    & \leq \frac{1}{ 1-\gamma } \delta_0 NK + \frac{1}{ 1-\gamma } \delta_0 NK = \frac{1}{ T(1-\gamma)},
\end{align*}
where the last inequality holds when $ T\log(1/\gamma)  \geq  \log(2NKT)$.  

On the other hand, we have 
\begin{align*}
    V_t^\pi(s_t)  - \widetilde V_t^\pi(s_t) = \EE^{\pi} \Big [\sum_{i=0}^\infty \gamma^{t+i} (r (s_{t+i},a_{t+i}) - \widetilde r (s_{t+i},a_{t+i}) )  \Big| s_1,\cdots, s_t \Big] \geq 0.
\end{align*}
Therefore,
\[
    \reg(T) \leq \widetilde{\reg}  (T) + \sum_{t=1}^T [(V^*(s_t)  - \widetilde V^*(s_t) ) - ( V_t^\pi(s_t)  - \widetilde V_t^\pi(s_t))] \leq \widetilde{\reg}  (T) + \frac{1}{1-\gamma}.
\]
\end{proof}

\begin{proof}[proof of Lemma~\ref{lem:reg:decomp}] 
Define $V_t(s) = \max_a Q_t(s,a)$. By the assumption that $\widetilde Q^*(s,a)\leq Q_t(s,a)$, we have $\widetilde V^*(s) \leq V_t(s)$. Then 
\[
    \widetilde \Delta_t =  \widetilde V^*(s_t) - \widetilde V_t^\pi(s_t) \leq V_t(s_t) - \widetilde V_t^\pi(s_t) = Q_t(s_t,a_t) - \widetilde V_t^\pi(s_t),
\]
where the last equality holds since $a_t \in \argmax_a Q_t(s_t, a)$, i.e., we take the action $a_t$ greedily according to $Q_t$.

The optimal truncated Q-function $Q^*_{\widehat{\gT}_{t_s}, \widetilde r+b_{t_s}}(s,a)$ (Algorithm~\ref{algo:vit}) satisfies the following truncated Bellman equation:
\begin{equation}
    Q^*_{\widehat{\gT}_{t_s}, \widetilde r+b_{t_s}}(s,a) = \min \{\Lambda/(1-\gamma),  \widetilde r(s,a) + b_{t_s}(s,a) + \gamma \EE_{s'\sim \PP_{\widehat{\gT}_{t_s}}(\cdot|s,a)} \max_{a'} Q^*_{\widehat{\gT}_{t_s}, \widetilde r+b_{t_s}}(s',a') \}.  
\end{equation}
Then we have 
\begin{align*}
    \widetilde \Delta_t = & \min\Big \{\widetilde r(s_t,a_t)+b_{t_s}(s_t,a_t) + \gamma \EE_{s'\sim \PP_{\widehat{\gT}_{t_s}}(\cdot|s_t,a_t)} V_t(s') , \Lambda /(1-\gamma) \Big \} \\
     & -  \Big( \widetilde r(s_t,a_t) + \gamma \EE_{s' \sim \PP(\cdot|s_t,a_t)} \widetilde V_{t+1}^\pi (s')  \Big) \\
     \leq & \Big( \widetilde r(s_t,a_t)+b_{t_s}(s_t,a_t) + \gamma \EE_{s'\sim \PP_{\widehat{\gT}_{t_s}}(\cdot|s_t,a_t)} V_t(s')\Big) -  \Big( \widetilde r(s_t,a_t) + \gamma \EE_{s' \sim \PP(\cdot|s_t,a_t)} \widetilde V_{t+1}^\pi (s')  \Big) \\
     = & b_{t_s}(s_t,a_t) + \gamma \Big( \EE_{s'\sim \PP_{\widehat{\gT}_{t_s}}(\cdot|s_t,a_t)} V_t(s')-   \EE_{s' \sim \PP(\cdot|s_t,a_t)} \widetilde V_{t+1}^\pi (s') \Big ) \\ 
     = & b_{t_s}(s_t,a_t) + \gamma \Big( \EE_{s'\sim \PP_{\widehat{\gT}_{t_s}}(\cdot|s_t,a_t)} V_t(s')  -  \EE_{s' \sim {\PP}(\cdot|s_t,a_t)} V_t(s')\Big) \\
      & + \gamma \Big(  \EE_{s' \sim {\PP}(\cdot|s_t,a_t)} V_t(s')  -   \EE_{s' \sim \PP(\cdot|s_t,a_t)} \widetilde V_{t+1}^\pi (s') \Big ).
\end{align*}
For the second term, since $|V_t(s')|\leq \Lambda /(1-\gamma)$, by Lemma~\ref{lem:bonus}, we have 
\[
    \gamma \Big( \EE_{s'\sim \PP_{\widehat{\gT}_{t_s}}(\cdot|s_t,a_t)} V_t(s')  -  \EE_{s' \sim {\PP}(\cdot|s_t,a_t)} V_t(s')\Big) \leq b_{t_s}(s_t,a_t).
\]
For the last term, we have 
\begin{align*}
    & \gamma \Big(  \EE_{s' \sim {\PP}(\cdot|s_t,a_t)} V_t(s')  -   \EE_{s' \sim \PP(\cdot|s_t,a_t)} \widetilde V_{t+1}^\pi (s') \Big ) \\
    = &  \gamma \xi_t + \gamma \Big(V_t(s_{t+1}) - \widetilde V_{t+1}^\pi (s_{t+1})  \Big),
\end{align*}
where 
\[
\xi_t = \Big(  \EE_{s' \sim {\PP}(\cdot|s_t,a_t)} (V_t(s') - \widetilde V_{t+1}^\pi (s')) - (V_t(s_{t+1}) - \widetilde V_{t+1}^\pi (s_{t+1})) \Big ).
\]

Therefore, we have 
\begin{align*}
    \sum_{t=1}^T \widetilde \Delta_t &\leq \sum_{t=1}^{T} \Big(V_{t}(s_{t}) - \widetilde V_{t}^\pi (s_{t})  \Big) \\
    &\leq 2 \sum_{t=1}^T b_{t_s}(s_t,a_t) + \gamma \sum_{t=1}^T \xi_t + \gamma \sum_{t=1}^T \Big(V_t(s_{t+1}) - \widetilde V_{t+1}^\pi (s_{t+1})  \Big). 
\end{align*}
For the last term above, we have 
\begin{align*}
    &\qquad  \sum_{t=1}^T \Big(V_t(s_{t+1}) - \widetilde V_{t+1}^\pi (s_{t+1})  \Big) \\
    &= \sum_{t=1}^T \Big(V_{t+1}(s_{t+1}) - \widetilde V_{t+1}^\pi (s_{t+1})  \Big)  + \underbrace{\sum_{t=1}^T\Big( V_t(s_{t+1}) - V_{t+1}(s_{t+1}) \Big)} _{E_T}\\
    & = \sum_{t=0}^{T-1} \Big(V_{t+1}(s_{t+1}) - \widetilde V_{t+1}^\pi (s_{t+1})  \Big) + \Big[ V_{T+1}(s_{T+1}) -\widetilde V_{T+1}^\pi(s_{T+1}) - V_1(s_1) + \widetilde V_1^\pi(s_1) \Big] + E_T \\ 
    & \leq  \sum_{t=1}^{T} \Big(V_{t}(s_{t}) - \widetilde V_{t}^\pi (s_{t})  \Big) + 2\Lambda/(1-\gamma) + E_T.
\end{align*}

Notice that $\{\xi_t\}$ is a martingale difference sequence. Therefore, by Azuma-Hoeffding inequality, we have with probability at least $1-\delta$,
\[
    \sum_{t=1}^T \xi_t \leq    \frac{2 \Lambda}{1-\gamma} \sqrt{T\log\frac{1}{\delta}}.
\]

To summarize, we have 
\begin{align*}
    \widetilde \reg (\pi ) &\leq \sum_{t=1}^{T} \Big(V_{t}(s_{t}) -  \widetilde V_{t}^\pi (s_{t})  \Big) \\
    &\leq 2 \sum_{t=1}^T b_{t_s}(s_t,a_t) +  \frac{2\gamma\Lambda}{1-\gamma} \sqrt{T\log\frac{1}{\delta}} + \gamma   \Big(\sum_{t=1}^{T} \Big(V_{t}(s_{t}) - \widetilde V_{t}^\pi (s_{t})  \Big) + 2\Lambda/(1-\gamma) + E_T \Big),
\end{align*}
which implies
\begin{align*}
    \widetilde \reg (\pi ) \leq \frac{1}{1-\gamma}\Big[ 2 \sum_{t=1}^T b_{t_s}(s_t,a_t) +  \frac{2\gamma \Lambda}{1-\gamma} \sqrt{T\log\frac{1}{\delta}} + \gamma   \Big(2\Lambda/(1-\gamma) + E_T \Big) \Big].
\end{align*}

\end{proof}

\begin{proof}[proof of Lemma~\ref{lem:switching:error}]
On the one hand, we have 
    \begin{align*}
        \frac{\det(\Sigma_T)}{\det(\Sigma_0)} & \geq \prod_{s=1}^{M-1} \frac{\det(\Sigma_{t_{s+1}-1})}{\det(\Sigma_{t_s-1})}> 2^{M-1}.
    \end{align*}
On the other hand, we also have 
\begin{align*}
        \frac{\det(\Sigma_T)}{\det(\Sigma_0)} &= \det(\Sigma_0^{-1} \Sigma_T)  \leq \Big(\frac{\mathrm{Tr}(\Sigma_0^{-1}\Sigma_T)}{d} \Big)^d \\
        & = \Big(\frac{\mathrm{Tr}(\boldsymbol I + \lambda^{-1} \sum_{i,k,t} \phi_{i,k}^t (\phi_{i,k}^t)^\top )}{d} \Big)^d \\
        &\leq \Big( \frac{d +  NKTL^2/\lambda }{d}\Big)^d.
\end{align*}
    Therefore, $M < \frac{1}{\log 2} d \log  \Big( \frac{d +  NKTL^2/\lambda }{d}\Big) + 1$.
\end{proof}

%% file: appendix/aux_lemma.tex
\section{Auxiliary Lemmas}

\begin{lem}\label{lem:optimism0}
    Let $\PP(s'|s,a)$ and $\widehat \PP (s'|s,a)$ be two transition probabilities. Assume that $0 \leq r(s,a) \leq \Lambda$. Let $Q^*$ be optimal Q-function for MDP $\gM_{\PP,r}$. Let $\widehat Q^*$ be the optimal truncated Q-function for $\gM_{\widehat\PP,r+b}$, which satisfies the following equation
    \[
        \widehat Q^*(s,a) = \min\{\Lambda/(1-\gamma), r(s,a)+b(s,a) + \gamma \EE_{s'\sim \PP(\cdot|s,a)} \max_{a'} \widehat Q^*(s',a') \}.
    \]
    Then if 
    \[
        b(s,a) \geq \frac{\gamma\Lambda }{1-\gamma}  \| \PP(\cdot|s,a) - \widehat \PP(\cdot|s,a) \|_1,    
    \]
    we have 
    \[
        \widehat Q^*(s,a) \geq Q^*(s,a)  \text{ for all } (s,a).
    \]
    Furthermore, we have for any $V(s)$ such that $0\leq V(s) \leq \Lambda/(1-\gamma)$, 
    \[
        \gamma| \EE_{s'\sim \PP(s'|s,a)} V(s') - \EE_{s'\sim \widehat \PP(s'|s,a)} V(s') | \leq b(s,a).
    \]
    \end{lem}

\begin{lem}[Factorization]\label{lem:factor}
If $\displaystyle \PP(\cdot|s,a) = \otimes _{i=1}^N  \PP_{i }(\cdot|s,a)$, $\widehat \PP(\cdot|s,a) = \otimes _{i=1}^N  \widehat\PP_{i }(\cdot|s,a)$, then 
\[
    \| \PP(\cdot|s,a) - \widehat \PP(\cdot|s,a) \|_1 \leq  \sum _{i=1}^N  \| \PP_{i }(\cdot|s,a) - \widehat \PP_{i }(\cdot|s,a) \|_1.
\]
\end{lem}

\begin{lem}[Lemma 11 in \citep{abbasi2011improved}]\label{lem:epl}
For any $\{x_t\}_{t=1}^T \subseteq \RR^d$, let $\Sigma_t = \lambda \boldsymbol I + \sum_{t=1}^T x_t x_t^\top$, then we have 
\[
    \sum_{t=1}^T (1 \wedge \|x_t\|^2_{\Sigma_{t-1}^{-1}})  \leq 2d\log\frac{d\lambda+TL^2}{d\lambda},
\]
where $L = \sup \|x_t\|_2$.
\end{lem}

\begin{lem}[Lemma 12 in \citep{abbasi2011improved}] \label{lem:det}
Let $A,B \in \RR^{d\times d}$ be two positive definite matrices and $A \succeq B$. Then for any $x \in \RR^d$, we have 
\[
    \|x\|_A^2 \leq \|x\|_B^2 \cdot \frac{\det(A)}{\det(B)}.    
\]
\end{lem}

%% file: appendix/appendix_experiments.tex
\section{Additional Experiments}

\begin{figure}
    \centering
   \includegraphics[width=1\linewidth]{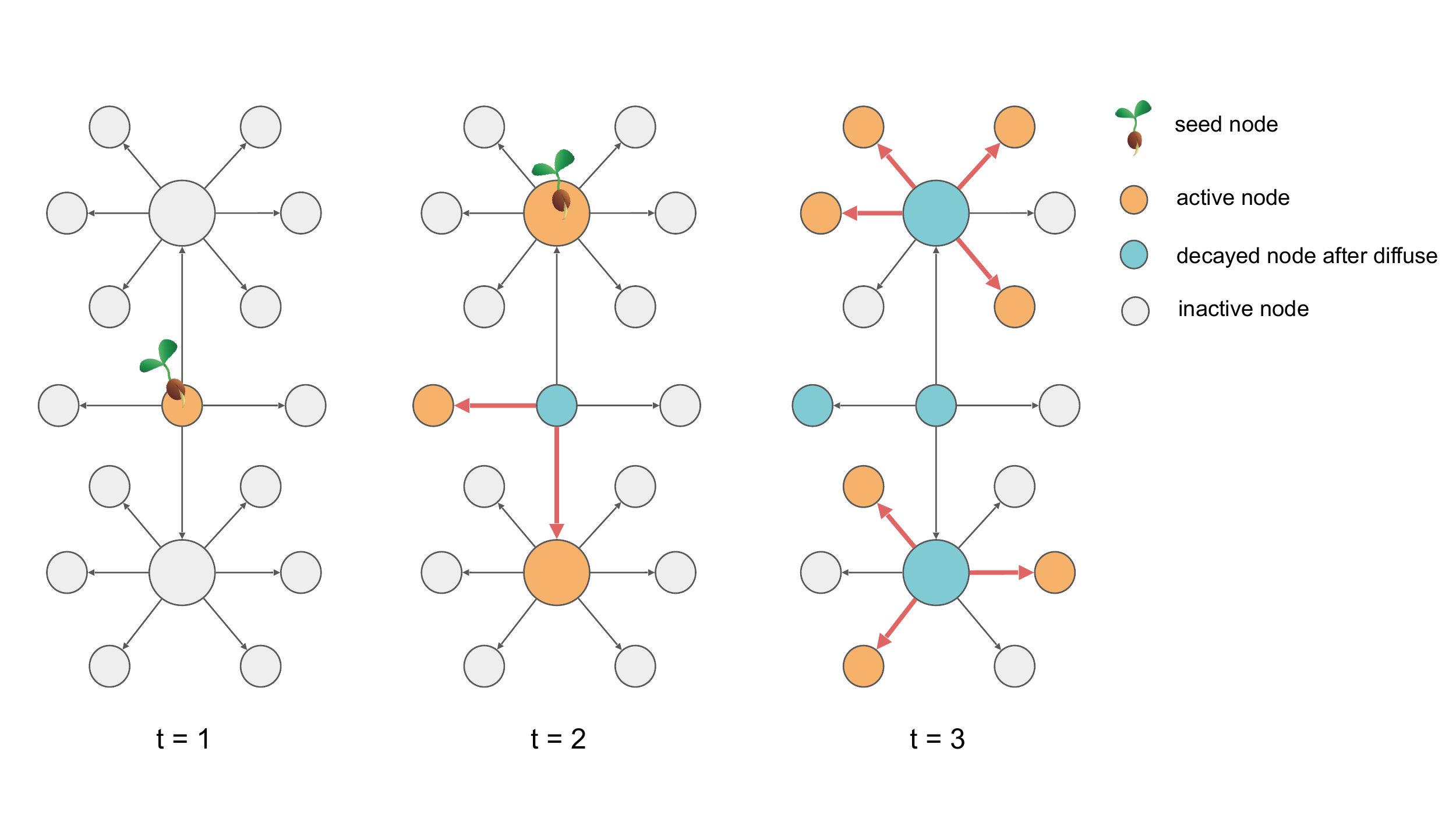}
    \vspace{-2mm}
    \caption{\textbf{Star-shape network.} The Synthetic network has three influential nodes; the center one can activate the other two. Only edges with positive probability are visualized. } 
\centering
\label{fig-bowtiegraph}
\end{figure}

We conduct additional experiments on another synthetic network as shown in Figure \ref{fig-bowtiegraph}, where a better action has more delayed reward and making decision at each time step is appreciated. This synthetic network has three influential nodes, while the center one is the best choice and has the ability to activate the other two influential nodes. Also the central influential node has delayed but higher expected reward than expected reward by activating either of other two neighboring influential nodes. For simplicity, we set only one content is available here, $K=1$ and $d_2=1$. The $d_1=6$ user feature vector has only one positive entry with value one, indicating its neighborhood subgraph. The discount factor of reward is $\gamma=0.9$.

In Figure \ref{fig-bowtiereward}, we compare the performance of our \textsc{MORIMA} to \textsc{MORIMA} with known $\Ab^k$ as upper bound and IMLinUCB as baseline. The details of these algorithms are the same as in section \ref{sec-experiment}. \textsc{MORIMA} exhibits its great power to explore the unknown graph efficiently; its learning curve at the first 40 time steps overlaps with \textsc{MORIMA} while knowing the true dynamics $\Ab^k$. In addition, the sum of cumulative rewards of \textsc{MORIMA} and the upper bound stay at the same high level. Furthermore, we observe that the learning procedure of IMLinUCB takes much longer and converges to a much lower level. The classic IM setting, activating $\textit{k}$ seeds at once for every $\textit{k}$ step, shows its limit while adaptive decision making leads to a better result.

We ran all experiments on our internal cluster with 8 CPUs, 128G memory per task.

\begin{figure}
    \centering
   \includegraphics[width=1\linewidth]{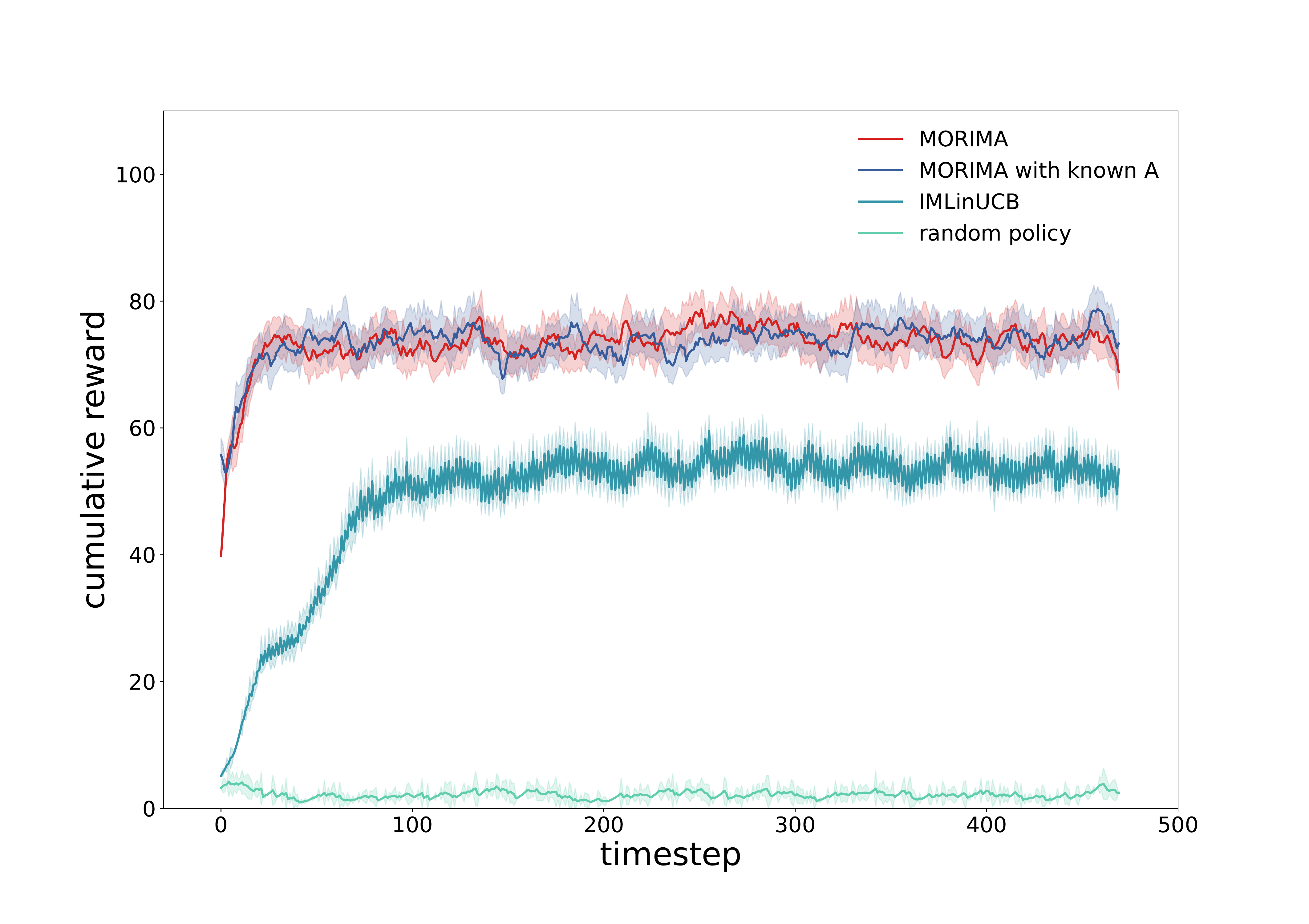}
    \vspace{-2mm}
    \caption{\textbf{Real-time discounted sum of rewards.} The Synthetic network (Fig.\ref{fig-bowtiegraph}) has static underlying dynamics $\gT^*$ generating diffusion probabilities from each node to all other nodes. Each line is averaged from 20 trajectories; 85\% CI bands are included.} 
\centering
\label{fig-bowtiereward}
\end{figure}

%% file: appendix/generalized_linear.tex
\section{Extension to Generalized Linear Model} \label{sec:generalized}

In this section, we show how to extend our algorithm and regret bound to generalized linear models. We first state below the modified assumption on the transition model.

\begin{assumption}[Generalized Bernoulli Independent Cascade Model\label{assump:generalized_bernoulli}] 
Let $s'$ be the next state. For each $k\in[K]$, we assume there is an underlying connectivity matrix $\Ab^k \in \mathbb{R}^{n \times n}$ such that
\begin{equation}
        \mathbb{P}(s_{i,k}' = 1|s) = \mu(\sum_{j} \Ab^k_{i,j} s_{j,k}),
    \label{eq:generalized-user-content-transition}
\end{equation}
where $\mu: \RR \to \RR$ satisfies $\mu(0)=0$ and $1/\kappa \leq \mu' \leq 1 $ for some $\kappa \geq 1$. And we assume $s_{i,k}'$'s are independent conditioned on $s$.
\end{assumption}

\subsection{\textsc{MORIMA} for Generalized Linear Model}
Next, we state the changes to our algorithm. Under this assumption, we cannot simply use ridge regression to get an empirical estimation of $\gT^*$. Instead, we estimate the tensor model by
\begin{equation}
   \widehat {\gT}_t = \argmin_{\gT} L_t(\gT) \equiv \Big \|  \lambda \gT  + \sum_{\tau=1}^{t-1} \sum_{k=1}^K \sum_{i=1}^N \big[ \mu (\langle \gT, \phi_{i,k}^\tau \rangle )  - (s_{\tau+1})_{i,k} \big] \phi_{i,k}^\tau \Big \|_{\Sigma_{t-1}^{-1}}, \label{eq:gen-That}
\end{equation}
where $\phi_{i,k}^t$ and $\Sigma_{t}$ are defined as in Section~\ref{sec:algorithm}. 

We still perform optimistic planning with respect to the truncated-reward model, where the reward bonus term $b_t(s,a)$ is replaced with 
\begin{equation}
    b_t(s,a) = \frac{4 \kappa \gamma \Lambda}{1-\gamma} \sum_{i=1}^N \sum_{k=1}^K (1 \wedge \beta_t \cdot \| \phi_{i,k} (s,a) \|_{\Sigma_{t-1}^{-1}} ) ,
    \label{eq:gen-bonus}
\end{equation}
which is the original bonus term multiplied by $2\kappa$. We adopt the same slow switching method as before.

\subsection{Regret Analysis}
We have the following regret bound, which is the original bound multiplied by $\kappa$.
\begin{theorem}\label{thm:gen-main}
    Let Assumption~\ref{assump:generalized_bernoulli}, Assumptions~\ref{assump:2}-\ref{decay} hold. With probability at least $1-\delta$, Algorithm~\ref{algo:1} satisfies the following regret upper bound:
    \[
            \reg (T) \leq \widetilde{\gO} \Big( \frac{\kappa}{\Delta ^2(1-\gamma)^2} \cdot \big(d\sqrt{C }/\Delta + \sqrt{dNK} \big) \cdot \sqrt{T}  \Big) + \mathrm{polylog}(T)\text{-}\mathrm{terms}, 
    \]
    where $d = \mathrm{dim} (\gT^*) = d_1^2d_2$.
\end{theorem}

\subsection{Proof Sketch of Theorem~\ref{thm:gen-main}}

The proof of Theorem~\ref{thm:gen-main} only differs from the proof of Theorem~\ref{thm:main} slightly. Next we examine through the proof of Theorem~\ref{thm:main} and state the corresponding lemmas in the generalized linear model setting.

First, we have exactly the same result for the high probability bounds for the number of active user-content pairs.
\begin{lem}[High probability bounds for the number of active user-content pairs]\label{lem:gen-hpbound:user}
    Let Assumption~\ref{assump:generalized_bernoulli}, Assumptions~\ref{assump:2}-\ref{decay} hold. For any possibly non-stationary policy $\pi$, with probability at least $1-\delta$, we have for all $t\geq 1$, 
    \[
        \|s_t\|_1 \leq \frac{2}{\Delta}(\frac{2}{\Delta}\log \frac{2 t^2}{ \delta} + 1).
    \]
\end{lem}

The next two lemmas justify the choice of the bonus term.

\begin{lem}[Confidence Region]\label{lem:gen-conf}
    Let Assumption~\ref{assump:generalized_bernoulli}, Assumptions~\ref{assump:2}-\ref{decay} hold. With probability at least $1-\delta$, we have for all $t \geq 1$, 
   \[
       L_t(\gT^*) \leq \beta_t,
   \]
   where $\beta_t$ is defined as in Eqn.~\eqref{eq:beta:t}.
\end{lem}

\begin{lem}[Optimism]\label{lem:gen-bonus}
Let Assumption~\ref{assump:generalized_bernoulli}, Assumptions~\ref{assump:2}-\ref{decay} hold. Set the bonus term to be 
   \[
       b_t(s,a) = \frac{4\kappa\Lambda\gamma}{1-\gamma} \sum_{i=1}^N \sum_{k=1}^K (1 \wedge \beta_t \cdot \| \phi_{i,k} (s,a) \|_{\Sigma_{t-1}^{-1}} ).
   \]
   Then with probability at least $1-\delta$, we have the optimistic condition $\widetilde{Q}^*(s,a) \leq Q_t(s,a)$ holds for all $t \geq 1$. 
   
   Furthermore, we have for any $V(s)$ such that $0\leq V(s) \leq \Lambda/(1-\gamma)$, 
   \[
       \gamma| \EE_{s'\sim \PP(s'|s,a)} V(s') - \EE_{s'\sim \PP_{\widehat{\gT}_t}(s'|s,a)} V(s') | \leq b_t(s,a).
   \]
\end{lem}

\begin{proof}[proof of Theorem~\ref{thm:gen-main}]
We can verify that Lemma~\ref{lem:sg}, Lemma~\ref{lem:reg:decomp}, and Lemma~\ref{lem:switching:error} also hold for the generalized linear model. Therefore, the exact same proof of Theorem~\ref{thm:main} applies with Lemma~\ref{lem:bonus} replaced with Lemma~\ref{lem:gen-bonus}. The result only differs by a factor of $\kappa$.
\end{proof}

\subsection{Deferred proofs of Lemmas}
\begin{proof}[proof of Lemma~\ref{lem:gen-hpbound:user}]
    Notice that we have assumed $\mu(0)=0$ and $\mu' \leq 1$. Therefore for $z>0$, $\mu(z) = \mu(z) - \mu(0) \leq z-0 = z$. Then 
    \[
        \EE [(s_{t+1})_{i,k}|s_t,a_t]  = \mu (\sum_j \Ab^k_{i,j} (s_{ta_t})_{j,k}) \leq \sum_j \Ab^k_{i,j} (s_{ta_t})_{j,k}.
    \]
    Then the result holds by applying the same proof of Lemma~\ref{lem:hpbound:user}.
\end{proof}

\begin{proof}[proof of Lemma~\ref{lem:gen-conf}]
    We have 
    \begin{align*}
        L_t(\gT^*) &= \Big\|  \lambda \gT^*  + \sum_{\tau=1}^{t-1} \sum_{k=1}^K \sum_{i=1}^N \big[ \mu (\langle \gT^*, \phi_{i,k}^\tau \rangle )  - (s_{\tau+1})_{i,k} \big] \phi_{i,k}^\tau \Big \|_{\Sigma_{t-1}^{-1}} \\
        &\leq \| \lambda \gT^*  \| _{\Sigma_{t-1}^{-1}} + \Big \|  \sum_{\tau=1}^{t-1} \sum_{k=1}^K \sum_{i=1}^N \big[ \mu (\langle \gT^*, \phi_{i,k}^\tau \rangle )  - (s_{\tau+1})_{i,k} \big] \phi_{i,k}^\tau \Big \|_{\Sigma_{t-1}^{-1}} \\
        &\leq \| \lambda \gT^* \| _{\lambda \boldsymbol I ^{-1}} + \Big \|  \sum_{\tau=1}^{t-1} \sum_{k=1}^K \sum_{i=1}^N \big[ \mu (\langle \gT^*, \phi_{i,k}^\tau \rangle )  - (s_{\tau+1})_{i,k} \big] \phi_{i,k}^\tau \Big \|_{\Sigma_{t-1}^{-1}} \\
        &\leq \sqrt{\lambda} \|\gT^*\|_2 + (\beta_t -  \sqrt{\lambda} \|\gT^*\|_2) \\
        &= \beta_t,
    \end{align*}
    where the last but one inequality holds by applying Lemma~\ref{lem:self-normed} with the same variance upper bound as in the proof of Lemma~\ref{lem:conf}.
    \end{proof}

\begin{proof}[proof of Lemma~\ref{lem:gen-bonus}]
    Let $\widehat\PP =  \PP_{\widehat{\gT}_t}$. By the assumption that $\mu' \leq 1$, we still have 
    \begin{align*}
        \| \PP_{i,k}(\cdot |s,a) - 
        \widehat\PP_{i,k}(\cdot |s,a) \|_1 &\leq 2(1 \wedge |\langle \gT^* - \widehat{\gT}_t , \phi_{i,k}(s,a) \rangle |)\\
        &\leq 2(1 \wedge \|\gT^* - \widehat{\gT}_t \|_{\Sigma_{t-1}} \cdot \|\phi_{i,k}(s,a)\|_{\Sigma_{t-1}^{-1}}).
    \end{align*}
    By Lemma~\ref{lem:gen-conf}, with probability at least $1-\delta$, we have $L_t(\gT^*) \leq \beta_t$ and then $L_t(\widehat{\gT}_t) \leq \beta_t$ for all $t\geq 1$. Therefore, we have
    \begin{align*}
        2\beta_t &\geq L_t(\widehat{\gT}_t) + L_t(\gT^*)   \\
        &\geq \Big \|    \lambda \widehat{\gT}_t  + \sum_{\tau=1}^{t-1} \sum_{k=1}^K \sum_{i=1}^N \big[ \mu (\langle \widehat{\gT}_t, \phi_{i,k}^\tau \rangle )  - (s_{\tau+1})_{i,k} \big] \phi_{i,k}^\tau 
            - \lambda \gT^*  - \sum_{\tau=1}^{t-1} \sum_{k=1}^K \sum_{i=1}^N \big[ \mu (\langle \gT^*, \phi_{i,k}^\tau \rangle )  - (s_{\tau+1})_{i,k} \big] \phi_{i,k}^\tau 
          \Big \|_{\Sigma_{t-1}^{-1}} \\
        &= \Big \| \Big[ \lambda \boldsymbol I +  \sum_{\tau=1}^{t-1} \sum_{k=1}^K \sum_{i=1}^N \big[ \mu' (\langle \widetilde {\gT}, \phi_{i,k}^\tau \rangle )   \big]  \phi_{i,k}^\tau (\phi_{i,k}^\tau)^\top \Big]   \cdot (\widehat{\gT}_t - \gT^*)  \Big \|_{\Sigma_{t-1}^{-1}} \\
        &\geq 1/\kappa \cdot\| \Sigma_{t-1}  \cdot  (\widehat{\gT}_t - \gT^*)\|_{\Sigma_{t-1}^{-1}} \\
        &= 1/\kappa \cdot \| \widehat{\gT}_t - \gT^* \|_{\Sigma_{t-1} },
    \end{align*}
    where we use Lagrange Mean Value Theorem in the first equality and $\mu'(\cdot) \geq 1/\kappa$. Then the desired result holds by applying Lemma~\ref{lem:factor} and Lemma~\ref{lem:optimism0} in the same way as the proof of Lemma~\ref{lem:bonus}.
\end{proof}